\documentclass[11pt]{article}

\usepackage[final]{acl}

\usepackage{times}
\usepackage{latexsym}

\usepackage[T1]{fontenc}
\usepackage[utf8]{inputenc}

\usepackage{microtype}

\usepackage{inconsolata}

\usepackage{graphicx}

\usepackage{amsmath,amsfonts,bm}

\def\eqref#1{equation~\ref{#1}}
\def\1{\bm{1}}

\DeclareMathAlphabet{\mathsfit}{\encodingdefault}{\sfdefault}{m}{sl}
\SetMathAlphabet{\mathsfit}{bold}{\encodingdefault}{\sfdefault}{bx}{n}

\usepackage[dvipsnames]{xcolor}
\definecolor{oi-blue}{HTML}{0072B2}
\definecolor{oi-orange}{HTML}{E69F00}
\definecolor{oi-green}{HTML}{009E73}
\definecolor{oi-red}{HTML}{D55E00}
\definecolor{oi-purple}{HTML}{CC79A7}
\definecolor{code-bg}{HTML}{FBFBFB}
\usepackage[most]{tcolorbox}
\usepackage{listings}
\usepackage{cleveref}
\usepackage{url}
\usepackage{algorithm} 
\usepackage{algpseudocode}
\usepackage{amsmath,amssymb}
\usepackage{booktabs}
\usepackage{amsthm}
\usepackage{caption}
\usepackage{enumitem}
\newtheorem{theorem}{Theorem}

\lstdefinestyle{papersafe}{
  captionpos=b,
  language=Python,
  basicstyle=\ttfamily\small,
  numbers=none,
  backgroundcolor=\color{code-bg},
  frame=single,
  rulecolor=\color{black},
  breaklines=True,
  breakatwhitespace=true,
  showstringspaces=false,
  keywordstyle=\bfseries\color{oi-blue},       % keywords
  commentstyle=\color{oi-green!60!black}, % comments
  stringstyle=\color{oi-red!85!black},         % strings
  identifierstyle=\color{black},               % names
  emph={np,torch,nn,plt,fit,predict},          % highlight common APIs
  emphstyle=\color{oi-purple},
  ndkeywordstyle=\color{oi-orange},            % (if you use morekeywords)
  aboveskip=2pt, belowskip=2pt, framesep=2pt, columns=fullflexible,
  prebreak=\mbox{\textcolor{gray}{$\hookleftarrow$}},
  postbreak=\mbox{\textcolor{gray}{$\hookrightarrow$}\space},
}
\title{BayesFlow: A Probability Inference Framework for Meta-Agent Assisted Workflow Generation}

\author{
Bo Yuan\textsuperscript{1,3}, Yun Zhou\textsuperscript{2}, 
{\bf Zhichao Xu\textsuperscript{2}}\\ {\bf Kiran Ramnath\textsuperscript{2}}, {\bf Aosong Feng\textsuperscript{2}}, {\bf Balasubramaniam Srinivasan\textsuperscript{2,3}} \\
\\
\textsuperscript{1}Georgia Institute of Technology \\
\textsuperscript{2}Amazon Web Services AI Lab \\
\textsuperscript{3}Correspondence to 
byuan48@gatech.edu, srbalasu@amazon.com
}

\begin{document}

\maketitle

\begin{abstract}
% As agentic tasks based on large language models (LLM) backbones become more sophisticated and require multistep executions with diverse tools, manual workflow design becomes prohibitively time consuming and error prone, motivating the need for \textbf{automated workflow generation}. Automatic workflow generation is the process of automatically synthesizing sequences of LLM calls, tool invocations, and post-processing steps for complex end-to-end tasks. Most prior methods cast this task as an optimization problem with limited theoretical grounding. We instead formulate it as Bayesian inference over a posterior distribution of workflows and introduce a general sampling framework, \textbf{Bayesian Workflow Generation (BWG)}, with rigorous theoretical analysis. We instantiate BWG with a new algorithm, \textbf{BayesFlow}, which focuses on exploration on high-reward regions while preserving convergence guarantees to the target distribution. Across six benchmark datasets, BayesFlow improves accuracy by up to 9\% over SOTA workflow generation baselines and by up to 65\% over zero-shot prompting.

% As agentic tasks based on large language models (LLM) backbones become more sophisticated and require multistep executions with diverse tools, manual workflow design becomes prohibitively time consuming and error prone, motivating the need for \textbf{automated workflow generation}.
Automatic workflow generation is the process of automatically synthesizing sequences of LLM calls, tool invocations, and post-processing steps for complex end-to-end tasks. Most prior methods cast this task as an optimization problem with limited theoretical grounding. We propose to cast workflow generation as Bayesian inference over a posterior distribution on workflows, and introduce \textbf{Bayesian Workflow Generation (BWG)}, a sampling framework that builds workflows step-by-step using parallel look-ahead rollouts for importance weighting and a sequential in-loop refiner for pool-wide improvements. We prove that, without the refiner, the weighted empirical distribution converges to the target posterior. We instantiate BWG as \textbf{BayesFlow}, a training-free algorithm for workflow construction. Across six benchmark datasets, BayesFlow improves accuracy by up to 9 percentage points over SOTA workflow generation baselines and by up to 65 percentage points over zero-shot prompting, establishing BWG as a principled upgrade to search-based workflow design. 

% \yun{please pick/modify one if needed: demonstrating the practical gains of Bayesian workflow inference, establishing BWG as a principled alternative to search-based workflow design}.
\end{abstract}

\section{Introduction}
Large Language Models (LLMs) have demonstrated remarkable generality, often solving tasks with a single carefully engineered prompt~\citep{wei2022chain, wang2022self, shinn2023reflexion, asai2024self}. However, as tasks become increasingly complex and demand multi-step reasoning, tool usage and memory, there has been a growing shift toward agentic systems, where multiple LLM agents collaborate through structured workflows, consistently achieving significant performance gains over simple prompt engineering.
~\citep{du2023improving,zhao2023competeai,liang2023encouraging,jiang2023llm,madaan2023self}. 

However, designing and refining workflows using frameworks such as AutoGen~\citep{wu2024autogen}, CAMEL~\citep{li2023camel}, and MetaGPT~\citep{hong2023metagpt} remains a labor-intensive and expertise-driven process, often involving extensive trial and error. 
This manual bottleneck limits the scalability of LLM systems to new tasks, each demanding bespoke solutions which hinders the adaptability of existing solutions to evolving task requirements.
~\citep{tang2023verifai, qiao2024benchmarking, cemri2025multi}. Consequently, \textbf{automatic workflow design} has emerged as a crucial research challenge to advance the capabilities and generality of LLM-based systems~\citep{zhuge2024gptswarm, zhang2024aflow, hu2024automated, li2024autoflow}.

% \bs{This para can be shrinked massively imo - let me know if thats OK - can work on it} \Bo{Sure. Please do it. Actually we need to shrink it. We are already reach the 8 page limit.}
Most prior methods formulate automatic workflow design as optimization problems, maximizing validation performance via search heuristics such as Monte Carlo tree search~\citep{zhang2024aflow}, linear heuristic search~\citep{hu2024automated}, or evolutionary strategies~\citep{li2025agentswift, shang2024agentsquare} via a meta optimizer LLM. However, these approaches typically lack rigorous theoretical foundations and yield a single high scoring solution with limited diversity. In contrast, our \textbf{sampling-based} Bayesian formulation provides principled guarantees and naturally produces a diverse set of high quality workflows via posterior sampling, an effect observed in controllable text generation~\citep{qin2022cold}, adversarial attacks~\citep{guo2024cold}, and machine translation~\citep{faria2024quest}.

% Our framework defines a broad design space within which existing methods correspond to special settings. 
In this work, we formulate workflow generation as a posterior sampling problem in structured \textbf{ code} representations~\citep{hu2024automated}, where each workflow is composed of modular code chunks corresponding to steps. We adopt a Bayesian sampling perspective~\citep{doucet2000sequential}, treating the meta optimizer LLM's internal knowledge as a prior distribution over plausible workflows, while incorporating external feedback signals (e.g. task-specific rewards or correctness) via an energy-based mechanism~\citep{du2019implicit}.
Energy-based models (EBMs) model unnormalized probability distributions and take the form
$\exp\!\big(R(s_{1:T})\big)$, where \(s_{1:T}\) denotes a workflow with \(T\) steps and \(R\) is a predefined reward function.

More precisely, our objective is to sample a complete workflow \(s_{1:T}\) from the unnormalized posterior.
\begin{equation}\label{eq:target}
    q(s_{1:T} | s_0) \propto p(s_{1:T}|s_0)\,\exp\!\big(R(s_{1:T})\big)
\end{equation}
where \(p(s_{1:T}|s_0)\) is a prior induced by the meta optimizer LLM $p$ and a query prompt $s_0$.
% This formulation naturally extends to multi-objective sampling. 
% EBMs support compositionality: When each reward term encodes a distinct objective, such as latency, cost, or effectiveness, combining the terms yields a combined reward whose maximizer jointly satisfies these objectives~\citep{mnih2005learning,haarnoja2017reinforcement}. Throughout, we focus on the validation performance and let \(R(s_{1:T})\) denote the validation set score. 
For notational simplicity, we will use \(s\) and \(s_{1:T}\) interchangeably and omit the dependency on $s_0$ when there is no confusion. It is worth noting that our posterior target distribution has rich connections to both reinforcement learning and inference scaling laws. Detailed analysis of connections is provided in Appendix~\ref{subsection:optimal}.

% \bs{this para below add way too much detail - unnecessary in the intro  - can be moved to appendix?}\Bo{Yes. I can work on it. I think it's one motivation why we consider Sampling, so I will summarize it here and move the details to appendix} \bs{Feel this is too much detail at this stage - probably move to Section 4, just as a para before Sec 4.1}

In this work, we introduce \textbf{Bayesian Workflow Generation (BWG)}, formalizing automatic workflow construction as Bayesian posterior sampling over workflows, replacing trajectory level code synthesis with fine-grained, step level generation via \emph{parallel look-ahead rollouts} for importance weighting and \emph{sequential in-loop refinement} for candidate improvement in Section~\ref{section:BWG}. Furthermore, we prove that the weighted empirical distribution converges asymptotically to the target distribution under mild assumptions in Theorem~\ref{theorem:1}. We instantiate BWG as \textbf{BayesFlow} in Section~\ref{section:BayesFlow} which demonstrates consistent gains across various benchmarks and model families, as shown in Section~\ref{section:Experiments}.

% We further instantiate BWG with a practical algorithm, \textbf{BayesFlow}. Experiments show that BayesFlow delivers consistently strong gains across reasoning benchmarks, multi-hop question answering, and challenging scientific datasets, on both closed-source and open-source LLMs.

In summary, our contributions are three-fold:
(1) Motivated by connections to reinforcement learning and inference scaling laws, we formulate workflow generation as Bayesian inference rather than optimization, which naturally promotes diversity in generated workflows. We also provide rigorous theoretical guarantees for this Bayesian formulation.
(2) We propose Bayesian Workflow Generation (BWG), a sampling-based framework that produces high-quality workflows via parallel look-ahead rollouts and sequential in-loop refinements. The parallel look-ahead step estimates the value of incomplete workflows without additional training or reliance on stronger closed-source models, while the in-loop refiner unifies prior methods under BWG and further improves workflow quality.
(3) We instantiate BWG with \textbf{BayesFlow} and demonstrate consistent performance gains in six datasets on both closed-source and open-source LLMs.
% \bs{Shall we add a line to say that as an implication of our contributions - there is a massive performance boost over SOTA?}\Bo{Added one sentence}
Consequently, BayesFlow improves accuracy by up to 9\% over SOTA workflow generation baselines on the math reasoning dataset and yields an average gain of up to $4.6\%$ across all benchmarks.
% \zx{1/ the introduction section is quite dense and technical heavy. Some of the materials should be covered in the algorithm section but instead is detailed in the introduction section, which creates a barrier for readers who only want to skim through the paper without going deep into algorithmic implementations. Can you better organize the introduction section for lighter reading experience? 2/ the current paper still have >0.5 page space left, would it be better to insert a related works section after Introduction? We want to highlight our differences between prior works, as well.} \Bo{Thank you for the suggestion! 1/ Yes I agree the introduction is very heavy. We will adjust it once all other components are done. We still need to insert figures, add one more group of experiments and add related work section. Actually we already have a very dense related work section and are trying to summarize it to two or three paragraphs}
% \bs{added a smaller related work section, and expanded on it in the appendix}

\section{Related Work}
Agentic workflows represent a fundamental paradigm in LLM applications, consisting of structured sequences of LLM invocations designed to solve complex tasks through predefined processes \citep{sirui2024meta, jia2024mobile}. 
Unlike autonomous agents that make dynamic decisions based on environmental feedback, agentic workflows operate through predetermined static logical sequences (or with conditionals) that can be systematically designed and refined \citep{zhuge2023mindstorms, wang2023voyager}. 
% General workflows focus on universal problem solving methodologies applicable across diverse domains\citep{wei2022COT, wang2022COTSC,madaan2023self,wang2023multipersona} whereas domain specific workflows are tailored to particular areas of application \citep{tal2024code, dong2024agent,hong2024data, yu2024hai, bo2024nlsql, yi2024gen,zhong2024achieving,zhou2024language}. 
While effective, these manually designed workflows require substantial human expertise, limiting their scalability to new domains.

% \subsection{Automatic Agentic Workflow Generation}
Recent research automates workflow optimization through prompt optimization \citep{fernando2023promptbreeder, wang2023promptagent, yang2023large, yuksekgonul2024textgrad}, hyperparameter tuning \citep{saad2024archon}, and workflow structure optimization \citep{li2024autoflow, hu2024automated, zhuge2024gptswarm}. 
% {\bf Training Free Approaches:} 
Training free methods leverage pretrained LLMs for iterative improvement without parameter updates. Component level methods like DSPy~\citep{khattab2023dspy} and TextGrad~\citep{yuksekgonul2024textgrad} optimize prompting strategies within fixed structures. 
Structure generation methods including ADAS \citep{hu2024automated}, GPTSwarm \citep{zhuge2024gptswarm}, and AFLOW \citep{zhang2024aflow} explore complete workflow architectures, though they struggle to leverage historical data and often perform no better than random sampling.
The fundamental challenge lies in effectively navigating the vast search space of possible workflow configurations while maintaining computational efficiency and ensuring the generated workflows generalize across different tasks and domains. 
% \yun{pls polish: Moreover, there is limited theoretical discussion on the property of generated workflows.} 
Moreover, there is limited theoretical discussion on the property of generated workflows. An extended discussion can be found in \Cref{app:extended_related_work}.

\section{Preliminaries}

\paragraph{Workflow generation as step-level code generation} We view automatic workflow generation as Bayesian posterior sampling in an energy-based model, and frame it concretely as a code generation task executed by an optimizer LLM. At the finest granularity, the optimizer LLM generates tokens, yet throughout the rest of this paper, we operate at step level: tokens are grouped into meaningful code chunks \(s_t\). To achieve so, we explicitly require the meta optimizer LLM to generate comments in the form of "Step <n>:" before every major block. See example workflows in Appendix~\ref{appendix:case_study} for each generated comment.
Although a workflow is originally represented as a directed graph~\citep{khattab2023dspy}, encoding it as code~\citep{zhang2024aflow} imposes a convenient linear order greatly simplifies generation processes.

\paragraph{Math formulation} A workflow with $T$ steps is a discrete trajectory ${s}_{1:T}=(s_1,\dots,s_T)$ drawn from the autoregressive prior of the meta optimizer LLM, $p({s}_{1:T}) = p(s_1)\prod_{t=2}^{T} p\!\left(s_t \mid {s}_{1:t-1}\right)$.
We use a single reward that measures validation accuracy
$R({s}_{1:T}) = \mathrm{Acc}_{\mathrm{val}}({s}_{1:T}) \in [0,1]$.
The posterior distribution over the workflows is then
$\frac{1}{Z}\, p( {s}_{1:T}) \, \exp\!\big[R( {s}_{1:T})\,\big]$,
where
$ Z =\sum_{ {s}_{1:T}} p( {s}_{1:T}) \exp\!\big[ R( {s}_{1:T})\,\big] $
is the normalization constant. For notational simplicity in mathematical expressions, we adopt a fixed horizon $T$. 
Workflows may terminate at different step counts; shorter sequences can be padded with empty strings, so that every trajectory admits a length representation $T$ without loss of generality.

\paragraph{Challenges} It is worth noting that sampling in this setting is nontrivial. The formulation exploits the autoregressive factorization of the LLM prior. Each conditional $p(s_t | s_{1:t-1})$ is produced step by step while operating under a regime \emph{terminal-reward} in which the reward $R(s_{1:T})$ is observed only after the entire workflow has been generated. The absence of intermediate rewards poses a significant challenge in guiding prefix decisions. The posterior step distribution can be written as
$$ 
p(s_t \mid s_{1:t-1}) \,
\underbrace{\sum_{s_{t+1:T}} p(s_{t+1:T}\!\mid s_{1:t}) \,\exp\!\big[ R(s_{1:T})\big]}_{\text{look-ahead value}}.
$$
This look-ahead term couples the current choice with an intractable sum over all completions, making the exact sampling challenging.

One way to solve this problem is to draw $N$ complete trajectories $s_{1:T}^{(i)} \sim p(s_{1:T})$ and then resample following unnormalized weights using only the reward:
$
w_i = \exp\!\big[ R(s_{1:T}^{(i)})\big], 
\bar w_i = \frac{w_i}{\sum_{j=1}^N w_j}.
$
This method is simple but can suffer from \emph{weight degeneracy}~\citep{johansen2009tutorial}, where a few
$\bar w_i$ dominate and most samples contribute negligibly. A typical strategy to mitigate it is step-level resampling. This motivates us to propose parallel look-ahead rollouts in Section~\ref{section:parallel}.

\section{Bayesian Workflow Generation}
\label{section:BWG}
To address the challenges in the terminal reward setting, we introduce \textbf{Bayesian Workflow Generation (BWG)}, a general Bayesian sampling framework shifting workflow construction from monolithic code synthesis to fine-grained step-level generation through \textbf{parallel loop-ahead rollouts} and \textbf{sequential in-loop refinements}. A class of inference-only methods treats workflow design as generating code at the trajectory level~\citep{zhang2024aflow,hu2024automated}. Step-wise generation builds workflows incrementally, but generally relies on strong closed source to score partial workflows~\citep{li2025agentswift} or an auxiliary learned value model to guide the search~\citep{shang2024agentsquare}. In contrast, BWG assigns importance weights via parallel look-ahead rollouts that anticipate the downstream energy of each partial workflow. This look-ahead mechanism removes the reliance on external proprietary models and avoids expensive process-model training, while ensuring that the weighted empirical distribution converges asymptotically to the target distribution in Theorem~\ref{theorem:1}.

In addition to parallel sampling, BWG integrates a sequential in-loop refiner that improves complete rollouts: Given a set of existing full workflows, the refiner is expected to produce additional higher-quality workflows. This module coincides with the core optimization step in previous work - implemented by MCTS in~\citep{zhang2024aflow,li2025agentswift} or by linear heuristic search in~\citep{hu2024automated}. This perspective casts BWG as a unifying framework that subsumes existing workflow generation methods as special cases.

\begin{figure*}
  \centering
  \includegraphics[width=\textwidth]{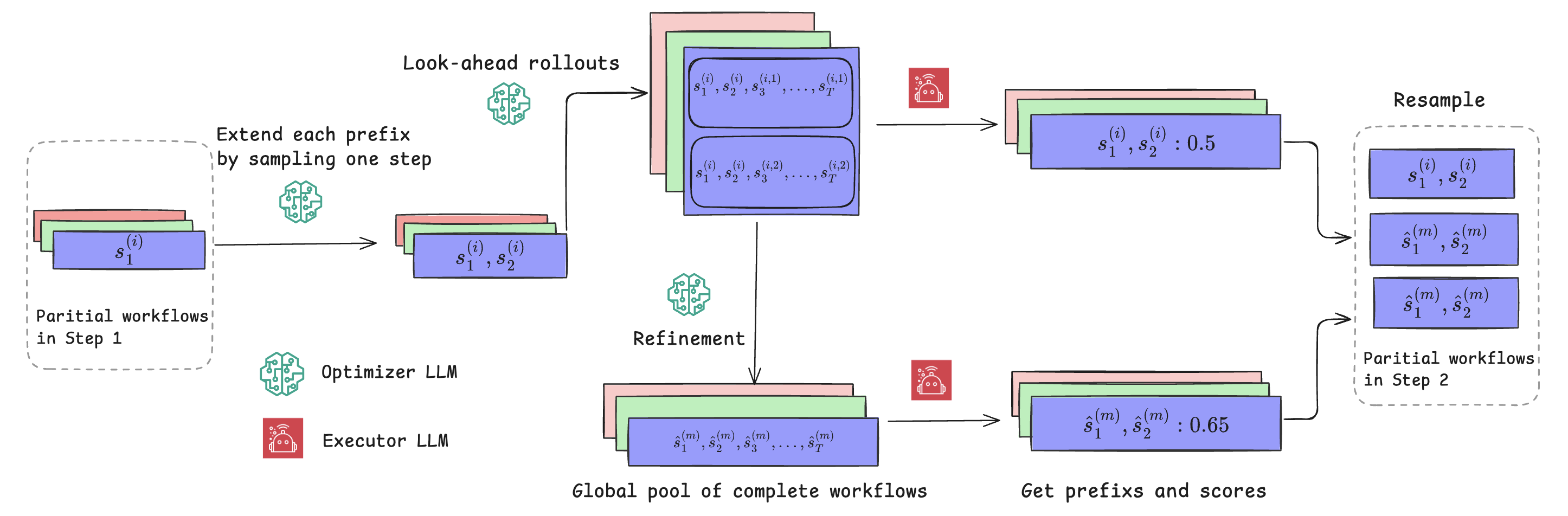}
  \caption{This diagram depicts Algorithm~\ref{alg:bayesflow_step} in a setting with \(N=3\) candidate partial workflows, \(K=2\) look-ahead rollouts per candidate, and \(M=3\) refinement attempts. The optimizer LLM is invoked to (i) sample next-step expansions, (ii) generate parallel look-ahead rollouts to estimate downstream value, and (iii) refine the current pool; the executor LLM is invoked to score complete workflows. The figure shows the transition of partial workflows from Step 1 to Step 2.}
  \label{fig:diagram}
\end{figure*}

\subsection{Parallel look-ahead rollouts}\label{section:parallel}

Unlike training a process reward model~\citep{li2025agentswift}, which is usually unstable~\citep{xu2025hybrid} and sensitive to distribution shift~\citep{levine2023baseline}, we adopt the idea from Sequential Monte Carlo (SMC)~\citep{doucet2001introduction, loula2025syntactic} and build the trajectory in the following round. 
\begin{enumerate}[noitemsep]
  \item \emph{Extend each prefix.} For each partial workflow $s_{1:t-1}^{(i)}$, sample one step from the prior $s_t^{(i)} \sim p(s_t | s_{1:t-1}^{(i)})$, forming the new prefix $s_{1:t}^{(i)}$.
  \item \emph{Look-ahead scoring.}\label{item:score} For each new prefix $s_{1:t}^{(i)}$, draw $K$ stochastic completions $\tilde s_{t+1:T}^{(i,k)} \sim p(s_{t+1:T} | s_{1:t}^{(i)})$ and set the importance weight $w^{(i)} = \frac{1}{K}\sum_{k=1}^K \exp\!\big[ R\big(s_{1:t}^{(i)}, \tilde s_{t+1:T}^{(i,k)}\big)\big]$.
  \item \emph{Normalize and Resample.}\label{item:resample} Compute normalized weights $\bar w^{(i)} = w^{(i)} \big/ \sum_{j=1}^N w^{(j)}$. Resample the prefixes $s_{1:t}^{(i)}$ according to $\bar w^{(i)}$; specifically, draw the $N$ indices $a_1,\dots,a_N$ with $\Pr(a_n=i)= \bar w^{(i)}$, thus replicating higher-weight prefixes and pruning lower-weight ones.
\end{enumerate}

In each round, we progress from $N$ partial workflows of length $t$ to $t{+}1$. Iterating this round until $T$ yields an approximation of the target posterior $p(s_{1:T}) \exp[R(s_{1:T})]$, mitigating weight degeneracy relative to pure prior reweighting. This mechanism leads to a tractable estimate of the intractable marginal and yields informative weights before resampling. We emphasize that the look-ahead rollouts are parallel across all partial workflows, incurring only negligible additional wall-clock time.

% BayesFlow preserves the particle--resample skeleton of Sequential Monte Carlo while adding two mechanisms tailored to \emph{terminal-reward} settings:  
% (i) a full \emph{look-ahead rollout} that estimates the downstream energy of each explorer prefix, and  
% (ii) a sequential \emph{Refiner} that upgrades these complete rollouts into higher-quality complete trajectories.  

% \paragraph{Look-ahead rollout} In terminal–reward settings the ideal intermediate target at step \(t\) is the
% marginal posterior
% \(
%   \pi_t(\mathbf{s}_{1:t})
%   \propto
%   \sum_{\mathbf{s}_{t+1:T}}
%   p(\mathbf{s}_{1:T})\,
%   \exp\bigl[E(\mathbf{s}_{1:T})\bigr],
% \)
% but evaluating this sum is infeasible: the reward is known only after the
% entire trajectory finishes, so every prefix must be marginalised over
% \emph{all} possible continuations.  Falling back to the autoregressive prior
% as the proposal leaves the importance weights extremely skewed and SMC quickly
% degenerates.  BayesFlow circumvents this obstacle with its first mechanism—a
% Monte-Carlo \emph{look-ahead rollout} that samples a handful of futures for
% each prefix, producing a tractable estimate of the intractable marginal and
% yielding informative weights before resampling.

\subsection{Sequential in-loop refinements} 

Even with informative look-ahead weights, exploration rollouts are drawn solely from the prior, limiting reach to workflows far from familiar patterns. More importantly, once a low-quality prefix is chosen, the parallel look-ahead has no mechanism to repair that early mistake: the prefix is fixed, and the trajectory tends to a poor terminal reward. Repeated resampling can partially mitigate this by reallocating mass to successful workflows, but a sample-only scheme remains highly reactive and underuses the self-reflection capabilities of the meta optimizer LLM~\citep{renze2024self}.

Therefore, after computing the look-ahead scoring (Stage~\ref{item:score}) in the parallel look-ahead mechanism, we augment the candidate pool \emph{before} the final normalization and resampling (Stage~\ref{item:resample}) via sequential in-loop refinements as follows: 

\begin{enumerate}[noitemsep]
  \item \emph{Global refinement operator.} For each particle $i$ with prefix $s_{1:t}^{(i)}$, let its $K$ rollouts be $\mathcal{C}_t^{(i)}=\{\tilde s_{t+1:T}^{(i,k)}\}_{k=1}^K$ and its look-ahead weight be $w^{(i)}$. Define the pool that includes all complete rollouts of $NK$ and their weights as $\mathcal{P}_t$. We introduce a general refinement operator $\mathcal{G}$ that generates $M$ complete workflows based on the entire pool, i.e., $\{\hat s_{1:T}^{(m)}\}_{m=1}^M \sim \mathcal{G}(\cdot | \mathcal{P}_t)$. Note that refinement operates at the workflow level: It may revise existing prefixes instead of merely extending them, yielding globally improved trajectories, and thereby enabling the refinement of earlier low-quality steps.

  \item \emph{Score proposals.} For each workflow, set $\hat w^{(m)}=\exp\!\big[R(\hat s_{1:T}^{(m)})\big]$.
  \item \emph{Normalize and resample.} Normalize all the weights in the $\{w^{(i)}\}_{i=1}^N \cup \{\hat w^{(m)}\}_{m=1}^M$ and resample $N$ prefixes from this enlarged pool.
\end{enumerate}

% \paragraph{Algorithm} Let \(K_1\) denote the number of \emph{explorer} particles drawn in parallel, and let \(K_2\) be the number of additional particles produced by the Refiner; the total particle budget is therefore \(K = K_1 + K_2\).  
% The resulting update step is detailed in Algorithm~\ref{alg:bayesflow_step}. Algorithm~\ref{alg:bayesflow_full} repeatedly invokes the step update of
% Algorithm~\ref{alg:bayesflow_step} until a step horizon
% $T$ is reached.
% Theoretical results suggests discarding earlier particles and
% propagating only the resampled pool, because—under infinite budget—the latest particles form the best Monte-Carlo approximation to the target distribution.  
% \emph{Practically}, however, the particle budget is small and every call
% to the LLM is expensive; we therefore log \emph{all} complete workflows
% generated along the way. This historical cache supplies two useful
% post-processing options::
% \emph{(i)} pick the single highest-reward trajectory
% $\arg\max_{s\in\mathcal{C}}E(s)$; or
% \emph{(ii)} execute several top-ranked workflows in parallel and apply
% a self-consistency vote over their outputs.  To ensure a fair
% comparison with baseline methods—each of which relies on a single
% workflow—we adopt strategy~(i) in all reported experiments.
% Nevertheless, strategy~(ii) incurs \emph{no} additional wall-clock time
% when workflows are run concurrently, and has the potential to further boost accuracy in practical deployments.

The refinement operator is a core, commonly used component in existing trajectory-level methods. We adopt a sequential generate-and-insert schedule: starting from the current pool, we produce one complete workflow at a time, immediately score it, and insert it back into the pool before generating the next. Our contribution here is a conceptual framework in which various methods can serve as a refiner. Concretely, previous work takes a pool of complete workflows with terminal scores and applies a search operator, e.g., linear heuristic expansion~\citep{hu2024automated}, Monte Carlo Tree Search (MCTS)~\citep{zhang2024aflow}, or evolutionary mutation~\citep{shang2024agentsquare}. In our implementation, the refiner is MCTS largely adopted from AFlow~\citep{zhang2024aflow} due to its superior performance across baselines in our experiments. 

Building on parallel look-ahead rollouts and sequential global refinement, we arrive at the Algorithm~\ref{alg:bayesflow_step}. See Figure~\ref{fig:diagram} for one concrete example. Its sole purpose in iteration $t$ is to advance the prefix population from length $t$ to $t{+}1$ by (i) exploiting parallel estimates of downstream value and (ii) injecting pool-wide refinements before a final resampling. This yields a transition operator preserving diversity, prioritizing promising prefixes, and preparing the pool for the next step.

\subsection{Bayesian Workflow Generation as a general framework}

BWG subsumes trajectory level workflow generation methods~\citep{zhang2024aflow,hu2024automated, shang2024agentsquare} as special cases. If the parallel look-ahead mechanism is disabled, BWG degenerates to pure refinement of complete workflows. In contrast, the removal of the refiner. i.e., $M=0$, yields pure parallel exploration akin to SMC. In the general case , BWG strictly extends these extremes by inserting step-level stochastic exploration before refinement, exploiting informative prefixes that purely trajectory-level schemes cannot access.

\begin{algorithm}[t]
\caption{\textsc{StepUpdate}}
\label{alg:bayesflow_step}
\begin{algorithmic}[1]
\Require Prefix pool $W_{t-1} := \{s_{1:t-1}^{(i)}\}_{i=1}^{N}$; rollout count $K$; number of refined workflows $M$

\State \textbf{Extend prefixes:} For each $i=1,\dots,N$, sample one step
$s_t^{(i)} \sim p(s_t | s_{1:t-1}^{(i)})$ and set $s_{1:t}^{(i)} \gets (s_{1:t-1}^{(i)}, s_t^{(i)})$.
\State \textbf{Look-ahead scoring:} For each new prefix $s_{1:t}^{(i)}$, draw $K$ completions
$\{\tilde s_{t+1:T}^{(i,k)}\}_{k=1}^{K} \sim p(s_{t+1:T} | s_{1:t}^{(i)})$ and compute the score $w^{(i)}= \frac{1}{K} \sum_{k=1}^{K} \exp\!\big[ R\big(s_{1:t}^{(i)}, \tilde s_{t+1:T}^{(i,k)}\big)\big]$.
\State \textbf{Global refinement:} Using the entire pool
$\mathcal{P}_t$, generate $M$ new complete workflows
$\{\hat s_{1:T}^{(m)}\}_{m=1}^{M} \sim \mathcal{G}(\cdot | \mathcal{P}_t)$.
\State \textbf{Project and score proposals:} For each $m$, denote the first $t$ steps as $\hat s_{1:t}^{(m)} $ and compute
$\hat w^{(m)} \;=\; \exp\!\big[ R\big(\hat s_{1:T}^{(m)}\big)\big]$.
\State \textbf{Augment pool:} Form $\mathcal{P}_t^{\mathrm{aug}}=\{(s_{1:t}^{(i)}, w^{(i)})\}_{i=1}^N \cup \{(\hat s_{1:t}^{(m)}, \hat w^{(m)})\}_{m=1}^{M}$.
\State \textbf{Resample prefixes:} Draw $N$ prefixes from $\mathcal{P}_t^{\mathrm{aug}}$ with probabilities proportional to their scores
$\{w^{(i)}\} \cup \{\hat w^{(m)}\}$;

\State \Return the resulting set as $W_t = \{\hat{s}_{1:t}^{(j)}\}_{j=1}^{N}$ and all generated complete workflows $ \mathcal{P}_t^{complete} = \{s_{1:t}^{(i)}, \tilde s_{t+1:T}^{(i,k)}\} \cup \{\hat s_{1:T}^{(m)}\}$
\end{algorithmic}
\end{algorithm}

\begin{algorithm}[t]
\caption{BayesFlow}
\label{alg:bayesflow_full}
\begin{algorithmic}[1]
\Require Rollouts count $K$; number of partial workflows in each round $N$, number of refined workflows $M$, maximum step count $T$
\State Initialize \emph{empty} prefix pool $W_0\gets\varnothing$ 
\State $t\gets 0$
\While{$t\le T$}
    \State Run $\textsc{StepUpdate}\bigl(W_t, K, M\bigr)$, and get $W_{t+1}$ and $ \mathcal{P}_{t+1}^{complete}$
    \State $t\gets t+1$
\EndWhile
\State $\mathcal{C}\gets \bigcup_{t=1}^T \mathcal{P}_{t}^{complete}$
\State \Return $\mathcal{C},\ \arg\max_{s\in \mathcal{C}} R(s)$
\end{algorithmic}
\end{algorithm}

\subsection{Theoretical analysis}

In this section, we present the convergence theorem that provides the theoretical analysis of Algorithm~\ref{alg:bayesflow_step}. To the best of our knowledge, this work provides the first theoretical analysis in the field of automatic workflow generation. The theorem shows that our parallel look-ahead design preserves asymptotic convergence to the original target distribution. It also implies that sampling from the target rather than the prior improves the expected reward of generated workflows. See Appendix~\ref{Appendix:proof} for a quick proof. It is worth noting that with the refinement mechanism enabled, the asymptotic convergence guarantee no longer holds. However, our ablation study demonstrates that refinement contributes to empirical gains.

\begin{theorem}\label{theorem:1}
Consider Algorithm~\ref{alg:bayesflow_step} without the refiner (i.e., $M{=}0$) for $T$ steps, and let
$\widehat\pi_{N,T}$ be the empirical distribution of the final $N$ complete workflows.
Then, as $N\!\to\!\infty$\footnote{Standard assumptions for SMC are taken to hold; see \citet{doucet2000sequential}.},
the empirical measure $\widehat\pi_{N,T}$ converges in probability to the target
distribution $q$ in~\eqref{eq:target}.
Moreover, $\mathbb{E}_{q}[R(s)] \ge \mathbb{E}_{p}[R(s)]$, where $R$ is the reward function and $p$ is the prior distribution of the meta  optimizer LLM.
\end{theorem}

Theorem~\ref{theorem:2} formalizes an ``exploration without drift'' property of BayesFlow with refinement. We show that refinement increasing diversity and encouraging broader
exploration of the workflow space, yet it does so in a controlled manner.
The TV bound quantifies this control: if each refinement step perturbs the induced complete workflow distribution
by at most $\varepsilon_t$, then the refined prefix flow stays close to the baseline flow, with deviation growing
only additively across steps. In particular, under a uniform bound $\varepsilon$, we obtain
$\mathrm{TV}(\pi_T,q)\le (T-1)\varepsilon$, meaning that refinement can expand the support of sampled workflows
while remaining provably close to the target distribution. See Appendix~\ref{Appendix:proof_2} for proof.

\begin{theorem}\label{theorem:2}
Consider Algorithm~\ref{alg:bayesflow_step} with the refiner enabled (i.e., $M>0$) for $T$ steps.
Let $\pi_T$ denote the limiting distribution over complete workflows.
Then $\pi_T$ admits the mixture decomposition
$
\pi_T \;=\; \alpha\, q \;+\; (1-\alpha)\,\pi_{\mathrm{rest}}
\quad
$ for some $\alpha\in(0,1]$,
where $q$ is the target distribution, and
$\pi_{\mathrm{rest}}$ is a distribution supported on workflows that undergo refinement at least once.

Moreover, let $\mu_t^0$ and $\mu_t^1$ denote the distributions over prefixes at step $t$ for Algorithm~\ref{alg:bayesflow_step}
without and with the refiner, respectively.
Assume that at each step $t$, for any prefix law, the distributions over complete workflows
before and after refinement differ by at most $\varepsilon_t$ in total variation distance.
Then the prefix drift is bounded as
\[
\mathrm{TV}(\mu_t^1,\mu_t^0) \;\le\; \sum_{k=1}^{t-1}\varepsilon_k,
\qquad t=1,\ldots,T.
\]
In particular, if $\varepsilon_k\le \varepsilon$ for all $k$, then
$
\mathrm{TV}(\mu_t^1,\mu_t^0)\le (t-1)\varepsilon,
$ and hence
$\mathrm{TV}(\pi_T,q)\le (T-1)\varepsilon$.
\end{theorem}

\section{BayesFlow}
\label{section:BayesFlow}
Building on Bayesian workflow generation, we instantiate \textbf{BayesFlow}, a flexible, inference-only method for automatic workflow generation. 

\paragraph{Design space}Inspired by the design space in \citet{nie2025weak}, BayesFlow avoids the predefined agentic modules used in previous work \citep{zhang2024aflow,hu2024automated,zhang2025multi} (e.g., \textsc{Ensemble}, \textsc{Revise}, \textsc{Programmer}), which can constrain the design space of the workflow. Instead, we specify only a lightweight workflow interface and expose generic helper primitives—LLM calls and code execution—leaving prompts, sampling hyperparameters, and control logic entirely to the meta optimizer, thereby encouraging innovation and adaptability. Concretely, we provide two helper functions: \texttt{chat\_completion}, which is a single LLM call given a system role, instructions, and a message history to return one or more candidate continuations; and \texttt{exec\_code}, which executes generated code against public unit tests and reports pass/fail outcomes with error diagnostics.  The \texttt{exec\_code} is only provided on coding datasets with public test sets. This minimal interface is model-agnostic, easy to extend, and decouples high-level workflow design from implementation details. See Appendix~\ref{appendix:helper} for detailed signatures. We also present the generated workflows in Appendix~\ref{appendix:case_study}.

\paragraph{Prior distribution} We instantiate the prior distribution $p$ of the meta optimizer LLM via a carefully designed prompt that instructs the model to produce workflow code in a fixed rule-based format. The prompt requires the agent to annotate each substantive step. If the resulting workflow fails to execute, we trigger a self-correction routine: the meta LLM is re-invoked with the same prompt augmented by the captured error log and asked to emit a corrected, runnable version. In this self-correction stage, we include the explicit prompt: \emph{“Do not use any try-except blocks. Fix the root cause rather than catching it.”} This mitigates reward hacking through exception handling and improves the success rate of self-correction.

% \paragraph{Reward function}
% The terminal reward $R(s)$ is, by default, the task accuracy measured on a held-out validation set, mirroring prior work in automatic workflow generation.  The formulation is, however, fully general: \(E(x)\) can be any scalar objective evaluated on the \emph{complete} workflow, including composite scores that trade off robustness to perturbations, API cost, latency constraints, or other application-specific metrics.  Multiple criteria are easily combined through the energy‐based view, simply by summing their individual energy terms.

\paragraph{Refiner implementation.}
Our refiner is inspired by the selection logic of MCTS in AFlow~\citep{zhang2024aflow} but uses a simpler pool-based procedure. 
From the current pool of complete workflows with associated validation rewards, we select the top-$C$ candidates ($C{=}3$ in all experiments), then sample one workflow with probabilities given by a softmax over rewards with temperature 0.1. We then pass the selected workflow to the meta optimizer LLM with the text gradient procedure of~\citet{yuksekgonul2024textgrad}. We design a prompt that requests one consequential edit while preserving the prescribed code format.
The edited workflow is inserted back into the pool, and the process repeats for the next refinement step.

\paragraph{BayesFlow algorithm}
We note that, in expectation, the mean reward for complete workflows does not decrease between steps, so the optimum population would be reached at the final step $T$. However, in practice, under finite sampling with a limited number of workflows per step, the highest-scoring workflow may appear earlier. Therefore, instead of restricting selection to the last step \(T\), we maintain a global pool of all candidates generated between steps and return as the final output the single workflow with the highest validation score observed throughout the run, as shown in Algorithm~\ref{alg:bayesflow_full}.

\section{Experiments}
\label{section:Experiments}
\subsection{Experiments setup}

\paragraph{Datasets and models}

We benchmark BayesFlow on six datasets across three domains, i.e., math reasoning: MATH~\citep{hendrycks2021measuring}, GSM8K~\citep{cobbe2021training}, question answering: DROP~\citep{dua2019drop}, HotpotQA~\citep{yang2018hotpotqa}, and expert knowledge: GPQA~\citep{rein2024gpqa}, MMLU-Pro~\citep{wang2024mmlu}. Workflows are drafted with Claude 3.5-Sonnet~\citep{anthropic2024claude35} (meta optimizer LLM) and executed with Claude 3.7-Sonnet~\citep{anthropic2025claude37} or
Qwen2.5-7B-Instruct~\citep{team2024qwen2}. All models are accessed via APIs. We set the temperature to 0 for the meta optimizer LLM. For MATH, we follow~\citep{hong2023metagpt} and select 617 level-5 problems spanning four categories: Combinatorics \& Probability, Number Theory, Pre-algebra, and Pre-calculus. For datasets with more than 1000 examples, we randomly sample 1000 instances. Then we adopt the standard 1:4 validation–test split used in prior work. All datasets are distributed with licenses that permit academic research use. 

\paragraph{Metrics}
For GSM8K and MATH we report the solve rate, the percentage of problems fully solved.  
For HotpotQA and DROP we use F1 score.   
For MMLU-Pro and GPQA, we measure accuracy, checking if the final selected option is correct.

\paragraph{Baselines} 

We evaluated six baselines in experiments. Three are prompt-based methods: (i) zero-shot prompts (IO), (ii) zero-shot chain-of-thought (CoT)~\citep{wei2022chain}, and (iii) CoT with self-consistency (CoT-SC)~\citep{wang2022self}, where each question is answered three times. The prompts are from the official repository of ADAS~\citep{hu2024automated}. We also compare against three automatic workflow generation baselines: ADAS, AFlow~\citep{zhang2024aflow} and MaAS~\citep{zhang2025multi}. For ADAS, we follow the original protocol with 30 iterations; for AFlow, we run 20 iterations and evaluate each generated workflow five times on the validation set; and for MaAS, we adopt the default hyperparameters from the official repository. We base our implementations on the official repositories where feasible. For datasets not originally supported, we extend the baselines by revising dataset-specific prompts while keeping the remaining components unchanged.

\subsection{Main results}

\begin{table*}[ht]
  \centering
  \renewcommand{\arraystretch}{0.95}
  \begin{tabular}{@{}lccccccc@{}}
    \toprule
    {Method} & {GSM8K} & {MATH} & {HotpotQA} & {DROP} & {MMLU-Pro} & {GPQA} & {Average} \\
    \midrule
    \multicolumn{8}{c}{\textit{Optimizer: Claude-3.5-Sonnet, Executor: Claude-3.7-Sonnet}} \\
    \midrule
    IO        & 90.0{\tiny$\pm$0.2} & 40.2{\tiny$\pm$0.2} & 57.4{\tiny$\pm$12.0} & 75.1{\tiny$\pm$0.1} & 57.8{\tiny$\pm$0.7} & 44.8{\tiny$\pm$0.5} & 60.9{\tiny$\pm$2.0} \\
    CoT       & 95.2{\tiny$\pm$0.2} & 43.2{\tiny$\pm$0.4} & 45.1{\tiny$\pm$17.9} & 81.5{\tiny$\pm$0.3} & 77.4{\tiny$\pm$1.4} & 65.6{\tiny$\pm$1.7} & 68.0{\tiny$\pm$3.0} \\
    CoT-SC    & 96.0{\tiny$\pm$0.2} & 45.0{\tiny$\pm$0.2} & 21.8{\tiny$\pm$0.0}  & 49.9{\tiny$\pm$0.5} & 24.3{\tiny$\pm$0.7} & \underline{65.8}{\tiny$\pm$1.2} & 50.5{\tiny$\pm$0.3} \\
    \midrule
    ADAS      & 96.0{\tiny$\pm$0.2} & 44.1{\tiny$\pm$2.4} & 73.3{\tiny$\pm$3.4}  & 81.2{\tiny$\pm$1.2} & 80.1{\tiny$\pm$1.1} & 64.0{\tiny$\pm$1.5} & 73.1{\tiny$\pm$0.8} \\
    MaAS      & \underline{96.4}{\tiny$\pm$0.2} & 41.3{\tiny$\pm$0.5} & \underline{76.3}{\tiny$\pm$0.6} & 84.2{\tiny$\pm$0.4} & \underline{82.0}{\tiny$\pm$0.6} & 55.2{\tiny$\pm$2.0} & 72.6{\tiny$\pm$0.4} \\
    AFlow     & \textbf{96.5}{\tiny$\pm$0.2} & \underline{60.1}{\tiny$\pm$1.5} & 63.7{\tiny$\pm$18.4}$^{\dagger}$ & \underline{89.2}{\tiny$\pm$0.3} & \textbf{82.3}{\tiny$\pm$0.7} & 65.3{\tiny$\pm$1.7} & \underline{76.2}{\tiny$\pm$3.1} \\
    BayesFlow (Ours) & 96.0{\tiny$\pm$0.3} & \textbf{69.4}{\tiny$\pm$4.1} & \textbf{77.5}{\tiny$\pm$0.7} & \textbf{90.8}{\tiny$\pm$1.5} & 81.8{\tiny$\pm$0.8} & \textbf{69.2}{\tiny$\pm$1.8} & \textbf{80.8}{\tiny$\pm$0.8} \\
    \midrule
    \multicolumn{8}{c}{\textit{Optimizer: Claude-3.5-Sonnet, Executor: Qwen 2.5-7B-Instruct}} \\
    \midrule
    IO        & 24.7{\tiny$\pm$0.2} & 16.6{\tiny$\pm$0.1} & 21.5{\tiny$\pm$0.1} & 27.7{\tiny$\pm$0.1} & 41.1{\tiny$\pm$0.1} & 33.1{\tiny$\pm$0.4} & 27.4{\tiny$\pm$0.1} \\
    CoT       & \textbf{90.6}{\tiny$\pm$0.3} & 17.6{\tiny$\pm$0.2} & 21.9{\tiny$\pm$0.4} & 32.2{\tiny$\pm$0.3} & 53.9{\tiny$\pm$0.4} & 34.4{\tiny$\pm$1.4} & 41.8{\tiny$\pm$0.3} \\
    CoT-SC    & 90.0{\tiny$\pm$0.3} & 15.6{\tiny$\pm$0.1} & 27.1{\tiny$\pm$0.7} & 12.3{\tiny$\pm$0.1} & 42.3{\tiny$\pm$1.5} & 34.0{\tiny$\pm$0.7} & 36.9{\tiny$\pm$0.3} \\
    \midrule
    ADAS      & 86.9{\tiny$\pm$0.5} & 21.6{\tiny$\pm$1.7} & \underline{68.3}{\tiny$\pm$0.8} & \underline{75.1}{\tiny$\pm$1.6} & 53.4{\tiny$\pm$1.0} & 34.1{\tiny$\pm$2.2} & 56.6{\tiny$\pm$0.6} \\
    MaAS      & 88.3{\tiny$\pm$0.1} & \underline{47.1}{\tiny$\pm$0.2} & 56.3{\tiny$\pm$0.1} & 48.7{\tiny$\pm$0.8} & \underline{54.6}{\tiny$\pm$0.2} & \underline{34.6}{\tiny$\pm$0.5} & 54.9{\tiny$\pm$0.2} \\
    AFlow     & \underline{90.4}{\tiny$\pm$1.1} & 47.0{\tiny$\pm$0.7} & 67.6{\tiny$\pm$0.7} & 69.9{\tiny$\pm$1.3} & 51.0{\tiny$\pm$3.6} & 33.5{\tiny$\pm$1.7} & \underline{59.9}{\tiny$\pm$0.7} \\
    BayesFlow (Ours) & \textbf{90.6}{\tiny$\pm$0.6} & \textbf{49.1}{\tiny$\pm$2.7} & \textbf{70.1}{\tiny$\pm$1.4} & \textbf{78.6}{\tiny$\pm$1.7} & \textbf{56.2}{\tiny$\pm$1.3} & \textbf{35.7}{\tiny$\pm$2.2} & \textbf{63.4}{\tiny$\pm$0.7} \\
    \bottomrule
  \end{tabular}
    \caption{Accuracy (mean~$\pm$~std.) of different methods across six datasets. We compare with both prompting baselines and agentic workflow methods. The best result for each dataset is in \textbf{bold}, and the second-best is \underline{underlined}. The high variance$^{\dagger}$ on AFlow for HotpotQA stems from occasional failures. }
  \label{tab:perf_transposed_final_reavg}
  %\zx{for some reasons I don't understand, this table is quite dense... can we improve its outlook}
\end{table*}

% \begin{table}[t]
%   \centering
%   \caption{Accuracy on the Claude-3.7-Sonnet executor.}
%   \label{tab:perf_claude}
%   % --- tighten columns/rows ---
%   \setlength{\tabcolsep}{3pt}        % default ~6pt
%   \renewcommand{\arraystretch}{0.92}  % default 1.0
%   \begin{tabular}{@{}lcccc@{}}        % @{} trims outer padding
%     \toprule
%     Dataset & BayesFlow & ADAS & MaAS & AFlow \\
%     \midrule
%     GSM8K      & 96.0{\tiny$\pm$0.3} & 96.0{\tiny$\pm$0.2} & 96.4{\tiny$\pm$0.2} & \textbf{96.5}{\tiny$\pm$0.2} \\
%     MATH       & \textbf{69.4}{\tiny$\pm$4.1} & 44.1{\tiny$\pm$2.4} & 41.3{\tiny$\pm$0.5} & 60.1{\tiny$\pm$1.5} \\
%     HotpotQA   & \textbf{77.5}{\tiny$\pm$0.7} & 73.3{\tiny$\pm$3.4} & 76.3{\tiny$\pm$0.6} & 63.7{\tiny$\pm$18.4}$^{\dagger}$ \\
%     DROP       & \textbf{90.8}{\tiny$\pm$1.5} & 81.2{\tiny$\pm$1.2} & 84.2{\tiny$\pm$0.4} & 89.2{\tiny$\pm$0.3} \\
%     MMLU-Pro   & 81.8{\tiny$\pm$0.8} & 80.1{\tiny$\pm$1.1} & 82.0{\tiny$\pm$0.6} & \textbf{82.3}{\tiny$\pm$0.7} \\
%     GPQA       & \textbf{69.2}{\tiny$\pm$1.8} & 64.0{\tiny$\pm$1.5} & 55.2{\tiny$\pm$2.0} & 65.3{\tiny$\pm$1.7} \\
%     \midrule
%     Average    & \textbf{80.8}{\tiny$\pm$1.5} & 73.1{\tiny$\pm$1.6} & 72.6{\tiny$\pm$0.7} & 76.2{\tiny$\pm$3.8} \\
%     \bottomrule
%   \end{tabular}
%     \vspace{2pt}
%     {\scriptsize $^{\dagger}$High variance stems from occasional failures}
% \end{table}

Each experiment was repeated three times with different random seeds and each selected workflow is evaluated on the test set five times (more hyperparameters shown in \Cref{Appendix:hyperparameters}). Table~\ref{tab:perf_transposed_final_reavg} reports the mean and standard deviation of accuracy. For the closed-source model, Claude 3.7-sonnet, BayesFlow achieves the highest average accuracy (80.8 \%) across the six benchmarks, outperforming the next‐best method (AFlow, 76.2 \%) by 4.6 percentage points.  
Its advantage is most pronounced on the challenging MATH and GPQA datasets, where it exceeds AFlow by 9.3 percentage points and 3.9 percentage points, respectively.
Standard deviations for \textbf{BayesFlow} are uniformly modest across the data sets, indicating a more stable behavior compared to other workflow baselines. For example, AFlow's high variance on HotpotQA is due to occasional failures to output effective workflows. BayesFlow also shows generalizability to the smaller-sized model, Qwen 2.5-7B-Instruct, achieving the top average accuracy ($63.4\%$), with a $3.5$ percentage points margin over AFlow ($59.9\%$) in the six tested datasets.

% In addition to the six datasets, we also test the performance of AFlow and BayesFlow on the coding dataset On claude 3.7-sonnet, MBPP~\citep{austin2021program}. Being different from the meta-propmt and helpful functions used in experiments on the six datasets. We also adjust the meta-prompt of the the otpmizer LLM by adding the helpful function \textsc{exec_code} which is used to test the writen code on public test sets. In the meta prompt we also emphare the exact output format to avoid syntax error. The implementation of \textsc{exec_code}  is largely adaopted from AFlow. The performance is 75.5 v.vs.  79.6 where the accuact stands fro pass@1. The final workflow obtained from Bayesflow contained six steps, as shown in 

In addition to the six main datasets, we evaluated AFlow and BayesFlow on the coding benchmark MBPP~\citep{austin2021program} using the Claude-3.7-Sonnet executor. Because code generation requires stricter interfaces, we modify the optimizer’s meta-prompt to (i) expose a helper function \texttt{exec\_code} that executes the function on public test cases and (ii) enforce an exact output format to minimize syntax errors. Notably, BayesFlow achieves pass@1 at $85.0\%$ in MBPP, which largely exceeds AFlow's $75.5\%$, further demonstrating the superior performance of BayesFlow\footnote{We had issues reproducing results for the other two baselines, thus omit MBPP from Table~\ref{tab:perf_transposed_final_reavg} now.}. 

The final MBPP workflow consists of six test-centered steps (see the complete workflow in Figure~\ref{fig:mbpp_workflow}). After generating candidate solutions, it performs early exit if all public tests pass; otherwise, it triggers self-correction guided by captured error logs, produces a revised implementation, and re-tests to verify fixes. A final fallback attempt provides an additional chance for recovery. Taken together, these steps form a reasonable and interpretable procedure that explicitly leverages public tests to drive refinement.

\subsection{Experiment analysis}

\paragraph{BayesFlow is more token-efficient}

\begin{table}
\centering
\normalsize % if it overflows in one-column, change to \small
\setlength{\tabcolsep}{6pt}
\renewcommand{\arraystretch}{1.2}
\begin{tabular}{l c c}
\hline
\textbf{Metric} & \textbf{AFlow} & \textbf{BayesFlow} \\
\hline

\multicolumn{3}{c}{\textit{MMLU-Pro, Optimizer: \textit{Claude Haiku 4.5}}} \\
\hline
Input tokens    & 38{,}557{,}324 & \textbf{5{,}290{,}460} \\
Output tokens   & 10{,}062{,}070 & \textbf{6{,}585{,}667} \\
Accuracy        & 0.432          & \textbf{0.571} \\
\hline

\multicolumn{3}{c}{\textit{DROP,  Optimizer: \textit{Claude Haiku 4.5}}} \\
\hline
Input tokens    & 79{,}170{,}565 & \textbf{14{,}527{,}334} \\
Output tokens   & 15{,}907{,}446 & \textbf{3{,}644{,}159} \\
F1              & 0.735          & \textbf{0.780} \\
\hline

\multicolumn{3}{c}{\textit{MMLU-Pro, Optimizer: \textit{Claude Sonnet 3.7}}} \\
\hline
Input tokens    & 63{,}200{,}652 & \textbf{45{,}630{,}330} \\
Output tokens   & \textbf{7{,}114{,}185} & 32{,}351{,}702 \\
Accuracy        & 0.535          & \textbf{0.573} \\
\hline

\multicolumn{3}{c}{\textit{DROP, Optimizer: \textit{Claude Sonnet 3.7}}} \\
\hline
Input tokens    & 114{,}158{,}960 & \textbf{26{,}616{,}937} \\
Output tokens   & 19{,}722{,}922  & \textbf{9{,}447{,}793} \\
F1              & 0.755           & \textbf{0.765} \\
\hline
\end{tabular}
\caption{\textbf{AFlow vs.\ BayesFlow: optimizer-side token usage and task performance.}
We report optimizer token usage  during workflow search using Claude Haiku 4.5 or Claude Sonnet 3.7 as the optimizer LLM and Qwen2.5-7B-Instruct as the executor LLM. \textbf{Bold} marks the better value within each dataset--optimizer block: lower token usage and higher task performance.}
\label{tab:method_compare_claude_stacked}
\end{table}

To further demonstrate the strength of BayesFlow, we compare it against AFlow, the second-best approach in terms of final accuracy in our experiments. As shown in Table~\ref{tab:method_compare_claude_stacked}, BayesFlow is more efficient than AFlow in most settings, achieving higher accuracy at lower cost. We emphasize that while Monte Carlo rollouts are generally more token-intensive, in automatic workflow generation the total token cost is driven not only by how many workflows are produced during training, but also by the cost of evaluating each workflow. This evaluation typically requires running the executor LLM over the full validation set, which can dominate overall usage.

We attribute BayesFlow’s efficiency to two key engineering choices. First, rather than relying on heavyweight, pre-defined operators such as LLM debate or self-consistency, we use a lightweight helper interface, substantially reducing overhead. Second, AFlow evaluates each workflow multiple times to obtain a stable estimate, whereas our implementation evaluates each workflow only once.

In terms of wall-clock runtime, BayesFlow can be further accelerated when sufficient API rate limits are available, since our rollouts can be generated and evaluated fully in parallel. In contrast, most search-based algorithms rely on sequential optimization steps that remain inherently serial, and thus cannot be sped up to the same extent even under sufficient rate limits.

\paragraph{BayesFlow is robust to 
the number of look-ahead rollouts.}

We study the effect of the ratio between the number of look-ahead rollouts $N$ and refined workflows $M$. To isolate the ratio,we fix the proposal budget $N+M$ per round. In the main experiments, the budget is 20, and here we vary $N\in\{5,10,15\}$, adjusting $M$ accordingly. The results in Table~\ref{tab:abl_n} using Claude-3.7-Sonnet executor show that BayesFlow is robust to the choice of $N$. This shows that our implementation is stable across exploration–exploitation settings. We also show that the expected reward almost monotonically increases over rounds in Appendix~\ref{Appendix:analysis} which supports the effectiveness of our look-ahead mechanism.

\begin{table}[t]
  \centering
  
  \renewcommand{\arraystretch}{0.90}
  \begin{tabular}{@{}lccc@{}}
    \toprule
    Dataset   & $N{=}5$                     & $N{=}10$                    & $N{=}15$                    \\
    \midrule
    GSM8K     & 96.3{\tiny$\pm$0.2}         & 96.0{\tiny$\pm$0.3}         & 96.5{\tiny$\pm$0.3}         \\
    MATH      & 70.4{\tiny$\pm$1.9}         & 69.4{\tiny$\pm$4.1}         & 71.1{\tiny$\pm$0.9}         \\
    HotpotQA  & 77.6{\tiny$\pm$0.1}         & 77.5{\tiny$\pm$0.7}         & 77.0{\tiny$\pm$0.7}         \\
    DROP      & 90.8{\tiny$\pm$0.2}         & 90.8{\tiny$\pm$1.5}         & 90.7{\tiny$\pm$0.3}         \\
    MMLU-Pro  & 82.8{\tiny$\pm$0.3}         & 81.8{\tiny$\pm$0.8}         & 82.1{\tiny$\pm$0.4}         \\
    GPQA      & 69.7{\tiny$\pm$1.6}         & 69.2{\tiny$\pm$1.8}         & 70.4{\tiny$\pm$1.2}         \\
    \bottomrule
  \end{tabular}
    \caption{Ablation over the number of partial workflows per round ($N$) with a fixed per round budget. Numbers are mean~$\pm$~standard derivation. 
  }
  \label{tab:abl_n}
\end{table}

\paragraph{BayesFlow demonstrates stronger workflow-level inference-time scaling.}
Unlike the standard inference-time scaling which repeatedly samples answers from a single LLM, we scale \emph{number of workflows} and aggregate their outputs. We report three simple metrics: \textbf{Best@L} counts an example correct if at least one of the $L$ workflows produces the right answer; \textbf{Mean@L} averages the fraction of correct answers among the $L$ outputs; and \textbf{Majority@L} takes a vote on the $L$ answers and is correct only if the most frequent answer matches the ground truth. 

In this experiment, we select the top $L$ workflows generated by both AFlow and BayesFlow. Workflows are produced using Claude-3.5-Sonnet as the optimizer LLM and Claude-3.7-Sonnet as the executor model. To demonstrate transferability, we then evaluated these fixed workflows with Claude-3.5-Sonnet on a randomly sampled 100-example subset of MATH. Across \(L\in\{1,2,4,8\}\), BayesFlow consistently outperforms AFlow, with especially pronounced gains in Best@L, as shown in Table~\ref{tab:workflow_scaling}. This pattern indicates that BayesFlow does not just find a single high-quality workflow; it produces a set of high-quality and diverse workflows, increasing the chance that at least one solves each problem. 

\begin{table}[t]
\centering
% \small
% \setlength{\tabcolsep}{6pt}
\renewcommand{\arraystretch}{0.71}
\begin{tabular}{lccc}
\toprule
\textbf{L} & {Majority@L} & {Mean@L} & {Best@L} \\
\midrule
$1$ & \textbf{74}/45     & \textbf{74}/45     & \textbf{74}/45 \\
$2$ & \textbf{74}/45     & \textbf{72.5}/27   & \textbf{80}/46 \\
$4$ & \textbf{74}/46     & \textbf{74}/30.75  & \textbf{85}/65 \\
$8$ & \textbf{78}/69     & \textbf{70.88}/47.12 & \textbf{86}/81 \\
\bottomrule
\end{tabular}
\caption{Accuracy (\%) shown as \textbf{BayesFlow}/AFlow, exemplified on MATH. The better value in each cell is bold.}
\label{tab:workflow_scaling}
\end{table}

\section{Conclusion}
We introduced Bayesian workflow generation, a principled framework that casts workflow synthesis as particle-based posterior sampling with look-ahead rollouts and text-gradient refinement. Across seven benchmarks and two executor families, BayesFlow delivers systematic gains: Surpassing the strongest baseline by $4.6$ percentage points on Claude-3.7-Sonnet and by $3.5$ percentage points on Qwen-2.5-7B on average. Taken together, these findings indicate that BWG is a promising direction compared to optimization-centric approaches. By emphasizing posterior sampling with lightweight refinement rather than heavy bespoke optimization, BayesFlow achieves competitive accuracy gains.

% \zx{can you mention some future works?}
% \zx{It is required to have a \#Limitations and Risks section for ARR submission, you can reference my prior work like \textit{Beyond Perplexity: Multi-dimensional Safety Evaluation of LLM Compression}} \Bo{Thank you! Sure we will add them}

\section*{Limitations}
The expressivity of BayesFlow partly stems from parallel look-ahead rollouts, which increase inference-time cost. A promising direction is a principled controller that dynamically tunes hyperparameters from on-line signals to reduce computation without hurting accuracy. Our current implementation prunes most rollout traces without long-term reuse. A stronger memory mechanism that retains and prioritizes high-value trajectories could better exploit complete look-ahead evidence and eliminate redundant sampling. Our evaluation focuses on two executor families in seven benchmarks but can be extended further; future work will extend the analysis to recent high-reasoning models and broader settings, e.g., multimodality or agents, to more rigorously assess the portability of sampling-based design and its budget–accuracy trade-offs.

\section*{Ethical considerations}

We will release the full codebase to enable  reproduction of our results. All datasets used are publicly available under their respective licenses; no human-subjects data, personally identifiable information, or proprietary assets are involved. We used AI tools only for minor language refinement (grammar, clarity, and formatting) during final manuscript preparation. No AI was used for idea generation, analysis, or text drafting, and all authors take full responsibility for the content.

\bibliography{ARR}
\newpage
\appendix
\section{Connection to reinforcement learning and inference scaling law}~\label{subsection:optimal}

 In reinforcement learning, a typical objective for an unknown distribution $q$ is
$
\max_{q} \mathbb{E}_{q}\!\big[R(s)\big]
-
\mathrm{KL}\big(q(s)\,\|\,p(s)\big)
$~\citep{schulman2017proximal,ouyang2022training} which is to maximize the reward while maintaining the KL divergence with respect to the prior distribution to mitigate model collapse~\citep{moalla2024no} and reward hacking~\citep{skalse2022defining}. 
One can show that the optimizer is exactly the target distribution in~\eqref{eq:target}; see Theorem~\ref{theorem:optimal} for a detailed proof.
On the inference side, weighted majority voting~\citep{wu2024inference} for LLM reasoning has been shown to be a scalable and effective way to improve performance; in essence, it corresponds to drawing and re-weighting samples in accordance with the same goal.
These insights naturally lead to a \textbf{principled} posterior sampling view of workflow generation.

\paragraph{Optimal solution in Reinforcement learning as posterior sampling.}
\begin{theorem}~\label{theorem:optimal}
Let \(p(s_{1:T})\) be a prior distribution, and let $R$ be a measurable reward function.
Consider the optimization over distributions \(q\):
\[
\max_{q}\; \mathcal{J}(q)
\;\;:=\;\; \mathbb{E}_{q}\!\big[R(s_{1:T})\big] \;-\; \mathrm{KL}\!\left(q\,\|\,p\right),
\]
where \(\mathrm{KL}(q\|p)=\mathbb{E}_{q}\!\left[\log \tfrac{q}{p}\right]\).
Then the unique maximizer is
\[
q^{\star}(s_{1:T}) \;=\; \frac{p(s_{1:T})\,\exp\!\big(R(s_{1:T})\big)}{Z}
\] where
\[
Z \;=\; \mathbb{E}_{p}[\exp\!\big(R(s_{1:T})\big)]
\]
\end{theorem}
\begin{proof}
Introduce a Lagrange multiplier \(\lambda\) for the normalization constraint \(\int q=1\) and form
\[
\begin{aligned}
\mathcal{L}(q,\lambda)
&= \int q(s)\,R(s)\,ds \;-\; \int q(s)\log\!\frac{q(s)}{p(s)}\,ds \\
&\quad + \lambda\!\left(1-\int q(s)\,ds\right).
\end{aligned}
\]
where we write \(s\equiv s_{1:T}\) for brevity.
Taking the first variation w.r.t.\ \(q\) and setting it to zero yields, for almost every \(s\),
\[
R(s) - \Big(\log q(s) - \log p(s) + 1\Big) - \lambda = 0
\]
Hence, we have
\[
\log q(s) \;=\; R(s) + \log p(s) - 1-\lambda,
\]
Exponentiating gives
\[
q(s) \propto p(s)\, e^{R(s)}.
\]

\end{proof}

\paragraph{Weighted majority voting as posterior sampling.}
Let the target posterior over workflows be \(q(s)\propto p(s)\exp(R(s))\), where \(p\) is a meta-optimizer prior and \(R\) is a reward oracle. Draw \(N\) i.i.d.\ samples \(\{s^{(i)}\}_{i=1}^N\sim p\) and form importance weights
\(w_i \propto \exp(R(s^{(i)}))\), normalized so that \(\sum_{i=1}^N w_i=1\).
Define the weighted empirical measure
\[
\hat q_N \;=\; \sum_{i=1}^N w_i\,\delta_{s^{(i)}}.
\]
If \(\mathbb{E}_{p}\!\left[e^{R(s)}\right]<\infty\), then by the law of large numbers for the sampling of self-normalized importance,
as \(N\) grows, the empirical distribution $\hat q_N$ concentrates on the target \(q\), and sampling by multinomial resampling indices with probabilities \(\{w_i\}\) yields draws asymptotically distributed as \(q\).

\section{Extended Related Work}
\label{app:extended_related_work}

\subsection{Agentic Workflows}
Agentic workflows represent a fundamental paradigm in LLM applications, consisting of structured sequences of LLM invocations designed to solve complex tasks through predefined processes \citep{sirui2024meta, jia2024mobile}. 
Unlike autonomous agents that make dynamic decisions based on environmental feedback, agentic workflows operate through static and predetermined logical sequences (or with conditionals) that can be systematically designed and refined \citep{zhuge2023mindstorms, wang2023voyager}. 
This structured approach enables workflows to leverage existing human domain expertise and iterative refinement processes, making them particularly suitable for tasks where consistent, reproducible outcomes are desired. 
% General workflows focus on universal methodologies like Chain-of-Thought \citep{wei2022COT}, Self-Consistency \citep{wang2022COTSC}, Self-Refine \citep{madaan2023self}, and MultiPersona \citep{wang2023multipersona}, while domain specific workflows target particular applications such as code generation \citep{tal2024code, dong2024agent}, data analysis \citep{hong2024data, yu2024hai, bo2024nlsql, yi2024gen}, mathematical problem solving \citep{zhong2024achieving}, and question answering \citep{zhou2024language}. 
% Despite their effectiveness, manual design requires substantial expertise, limiting scalability.
The landscape of agentic workflows can be broadly categorized into general and domain-specific approaches.

{\bf Generalist Agentic Workflows:} General workflows focus on universal problem-solving methodologies that can be applied across diverse domains. 
Notable examples include Chain-of-Thought prompting \citep{wei2022COT}, which guides LLMs through step-by-step reasoning processes; Self-Consistency \citep{wang2022COTSC}, which generates multiple reasoning paths and selects the most consistent answer; Self-Refine \citep{madaan2023self}, which iteratively improves outputs through self-critique; and MultiPersona approaches \citep{wang2023multipersona}, which leverage different perspectives to enhance problem-solving capabilities. 
Additionally, MedPrompt \citep{nori2023medprompt} demonstrates how structured prompting strategies can achieve expert-level performance in specialized domains through carefully designed workflow components.

{\bf Domain-Specific Agentic Workflows:} Domain-specific workflows are tailored to address the unique challenges and requirements of particular application areas. 
In code generation, workflows often incorporate testing, debugging, and iterative refinement processes, with notable examples including AlphaCodium \citep{tal2024code}, which transitions from prompt engineering to flow engineering, and AgentCoder \citep{dong2024agent}, which employs multi-agent collaboration for iterative code optimization. 
Data analysis workflows typically combine data exploration, statistical analysis, and visualization components \citep{hong2024data, yu2024hai, bo2024nlsql, yi2024gen}. Mathematical problem-solving workflows emphasize systematic reasoning, calculation verification, and solution validation \citep{zhong2024achieving}. 
Question-answering workflows often integrate information retrieval, fact verification, and answer synthesis processes, utilizing techniques such as Language Agent Tree Search to unify reasoning, acting, and planning \citep{zhou2024language}. 
While these manually designed workflows have demonstrated significant effectiveness in their respective domains, their development requires substantial human expertise and domain knowledge, limiting their scalability and adaptability to new problem areas.

\subsection{Automatic Agentic Workflow Generation}
The manual design of effective agentic workflows requires substantial human expertise and domain knowledge, creating a significant bottleneck for the widespread adoption of LLM-based systems across diverse applications. 
To address this limitation, recent research has focused on automating the discovery and optimization of agentic workflows, aiming to reduce human intervention while maintaining or improving performance. This emerging field encompasses various approaches that can be broadly categorized into three dimensions: automated prompt optimization \citep{fernando2023promptbreeder, wang2023promptagent, yang2023large, yuksekgonul2024textgrad}, hyperparameter tuning \citep{saad2024archon}, and complete workflow structure optimization \citep{li2024autoflow, hu2024automated, zhuge2024gptswarm}. 
The fundamental challenge lies in effectively navigating the vast search space of possible workflow configurations while maintaining computational efficiency and ensuring the generated workflows generalize across different tasks and domains. 
These automated approaches can be further distinguished by their optimization methodology: training free methods that rely on iterative refinement through LLM based search, and training based methods that learn to generate optimal workflows through supervised or reinforcement learning techniques.

{\bf Training Free Approaches:} Training free automated workflow optimization methods leverage the inherent capabilities of pre trained LLMs to iteratively improve workflow designs without requiring parameter updates. 
Early approaches have primarily focused on optimizing individual components within fixed workflow structures. DSPy~\citep{khattab2023dspy} provides a programming model that separates workflow logic from prompt engineering, allowing for systematic optimization of prompting strategies while maintaining workflow structure. 
TextGrad~\citep{yuksekgonul2024textgrad} treats prompt optimization as a form of automatic differentiation through text, enabling gradient-based optimization of prompting strategies.
PromptBreeder~\citep{fernando2023promptbreeder} introduces self-referential improvement mechanisms that evolve prompts through iterative refinement processes. 
These component-level optimization methods have demonstrated significant improvements in reasoning performance, but remain constrained by their focus on predefined workflow templates and limited exploration of alternative structural configurations.

More recent work has attempted to address the limitations of component-level optimization by exploring complete workflow structure generation. 
ADAS \citep{hu2024automated} represents workflows using code structures and employs linear heuristic search algorithms to discover effective configurations, though it faces efficiency challenges due to the accumulation of irrelevant information and coarse workflow storage mechanisms. 
GPTSwarm \citep{zhuge2024gptswarm} uses graph-based representations combined with reinforcement learning techniques for workflow optimization, but struggles with conditional states due to the limitations of graph structures in expressing complex logical relationships. 
AFLOW addresses these representational limitations by reformulating workflow optimization as a Monte Carlo Tree Search problem over code-represented workflows, where LLM-invoking nodes are connected by flexible edges that can express complex relationships including conditional logic, loops, and parallel execution \citep{zhang2024aflow}. 
However, training free methods fundamentally suffer from their inability to effectively leverage historical optimization data and environmental feedback, often performing no better than random workflow sampling and struggling with convergence to optimal solutions within limited iterations.

\textbf{Training Based Approaches:} Training based automated workflow generation and optimization methods address the limitations of training-free approaches by learning to generate optimal workflows through explicit parameter optimization. 
These methods can accumulate knowledge from previous optimization attempts and adapt their generation strategies based on empirical feedback. 
Weak-for-Strong Harnessing (W4S) introduces a novel paradigm where a smaller, trainable meta-agent is optimized via reinforcement learning to design workflows that effectively harness stronger language models~\citep{nie2025weak}. 
The approach formulates workflow optimization as a multi-turn Markov Decision Process and employs Reinforcement Learning for Agentic Workflow Optimization (RLAO), enabling the meta-agent to learn from both successful and failed workflow attempts through reward-weighted regression. 
This training based approach demonstrates superior performance compared to training free methods while maintaining cost efficiency by training smaller models rather than fine-tuning large ones.

ScoreFlow takes a different training based approach by leveraging preference optimization techniques adapted for workflow generation \citep{wang2025scoreflow}. 
The framework introduces Score DPO, a novel variant of Direct Preference Optimization that incorporates quantitative evaluation scores directly into the optimization process, addressing the variance and inaccuracies inherent in preference data that can slow convergence in standard preference optimization methods. 
ScoreFlow generates multiple workflows per task, evaluates their performance using quantitative metrics, and uses these scores to construct preference pairs for training the workflow generator through gradient-based optimization. 
This approach enables adaptive workflow generation where different workflows can be optimized for different types of problems, improving scalability compared to methods that optimize a single workflow for entire task sets. 
Both training based approaches demonstrate significant advantages over their training free counterparts, achieving superior performance while enabling smaller models to outperform larger ones through optimized workflow design, highlighting the potential of learned optimization strategies in automated agentic system development.

% \textbf{Concurrent Works:} We list the concurrent works which have appeared on Arxiv less than 3 months prior to our submission date, but do not benchmark against these methods.

\section{Quick proof of Theorem~\ref{theorem:1}}\label{Appendix:proof}

\begin{proof}

With $M{=}0$, Algorithm~\ref{alg:bayesflow_step} is an SMC sampler for the target
$p(s_{1:T}) \exp[R(s_{1:T})]$.
In step $t$, the mutation kernel is $p(s_t| s_{1:t-1})$ and the weight is
$\mathbb{E}_{p(s_{t+1:T} | s_{1:t})}\exp( R(s_{1:T}))$.
BWG uses the Monte Carlo estimator
$w^{(i)} = \frac1K\sum_{k=1}^K \exp\big( R(s_{1:t},\tilde s_{t+1:T}^{(k)})\big)$
with $\tilde s_{t+1:T}^{(k)} \sim p(\cdot| s_{1:t})$ i.i.d. Hence, the weight $w^{(i)}$ is
positive and unbiased. Then we can apply Theorem 3.5 in~\citet{rohrbach2022convergence} to prove the convergence of the equivalence distribution. To prove the inequality, let $Z(\beta)=\mathbb{E}_{p}[\exp(\beta R)]$ and the normalized distribution $q_\beta=\exp(\beta R)p/Z(\beta)$. Then $\frac{d}{d\beta}\,\mathbb{E}_{q_\beta}[R] \ge 0$.
Thus $\mathbb{E}_{q_\beta}[R]$ is non-decreasing in $\beta$, and for $\beta>0$,
$\mathbb{E}_{q_\beta}[R]\;\ge\;\mathbb{E}_{q_0}[R]\;=\;\mathbb{E}_{p}[R].$

\end{proof}

\section{Proof of Theorem~\ref{theorem:2}}\label{Appendix:proof_2}
\begin{proof}
For each final workflow, either it was never selected from a refined candidate at any step, or it was selected
from a refined candidate at least once. Conditioning on the event that refinement is never selected, the
resampling procedure reduces to the no-refiner algorithm, so the conditional distribution is exactly $q$.
Therefore, by the law of total probability, the final law admits the decomposition
$\pi_T=\alpha q+(1-\alpha)\pi_{\mathrm{rest}}$.

Let $\mu_t^0$ and $\mu_t^1$ be the prefix distributions at step $t$ without or with the refiner. Fix a step $t$. Consider the one-step update that maps a prefix distribution to the next-prefix distribution.
Introduce an intermediate distribution $\widetilde{\mu}_{t+1}$ defined as:
start from prefixes distributed as $\mu_t^1$, but perform the \emph{no-refiner} update for one step.
Then by the triangle inequality, $
\mathrm{TV}(\mu_{t+1}^1,\mu_{t+1}^0)
\;\le\;
\mathrm{TV}(\mu_{t+1}^1,\widetilde{\mu}_{t+1})
\;+\;
\mathrm{TV}(\widetilde{\mu}_{t+1},\mu_{t+1}^0).
$

For the first term, by assumption, for any prefix law the distributions over \emph{complete} workflows
before vs.\ after refinement differ by at most $\varepsilon_t$ in total variation.
The next-prefix distribution is obtained from a complete-workflow distribution by applying the projection map
that keeps only the first $t{+}1$ steps.
By pushforward contraction of total variation,
$\mathrm{TV}(\mu_{t+1}^1,\widetilde{\mu}_{t+1})
\;\le\;
\varepsilon_t.
$

For the second term, $\widetilde{\mu}_{t+1}$ and $\mu_{t+1}^0$ are obtained by applying the \emph{same}
 randomized update rule to inputs $\mu_t^1$ and $\mu_t^0$, respectively.
Total variation cannot increase under applying the same post-processing, hence
$\mathrm{TV}(\widetilde{\mu}_{t+1},\mu_{t+1}^0)
\;\le\; \mathrm{TV}(\mu_t^1,\mu_t^0).
$ Combining the two bounds yields the recursion
\[
\mathrm{TV}(\mu_{t+1}^1,\mu_{t+1}^0)
\;\le\;
\varepsilon_t + \mathrm{TV}(\mu_t^1,\mu_t^0).
\]
Hence
\[
\mathrm{TV}(\mu_t^1,\mu_t^0)\;\le\;\sum_{k=1}^{t-1}\varepsilon_k,\qquad t=1,\ldots,T.
\]
If $\varepsilon_k\le \varepsilon$ for all $k$, then $\mathrm{TV}(\mu_t^1,\mu_t^0)\le (t-1)\varepsilon$.

\end{proof}

\section{Hyperparameters}
\label{Appendix:hyperparameters}

Across all main experiments, we use a fixed set of hyperparameters: rollout count $K=1$, number of partial workflows per round $N=10$, number of refined workflows $M=10$, and maximum steps $T=5$. Note that the optimal configuration can vary by dataset: In \textsc{MATH} we observe consistent gains through $T=5$ rounds, whereas on other datasets three steps typically suffice to reach the best workflow. In addition, each complete workflow is evaluated on the validation set only once, while \textsc{AFlow} evaluates each workflow five times; although this increases variance, the sampling-based nature of our Bayesian framework largely mitigates it.

\section{More experiment analysis} \label{Appendix:analysis}

\paragraph{Expected reward almost monotonically increases over rounds.}

To assess the effectiveness of the parallel look-ahead mechanism, we plot the maximal, minimal, and average rewards at each round in Fig.~\ref{fig:val_curves}. To isolate the rollout mechanism, we disable sequential refinement by setting $M{=}0$ in these experiments. We observe that the mean validation accuracy improves monotonically or near-monotonically across rounds, indicating that look-ahead exploration alone is indeed effective. The greatest gain occurs between rounds~1 and~2 suggesting that early pruning of low-quality workflows accounts for most of the improvement. The min–max band also narrows markedly for MMLU-Pro, showing that the workflow pool becomes more homogeneous as unpromising branches are discarded. Overall, Fig.~\ref{fig:val_curves} demonstrates that parallel look-ahead can reliably guide step-level generation without the need of a process reward model.

\begin{figure}
  \centering
  \includegraphics[width=0.5\textwidth]{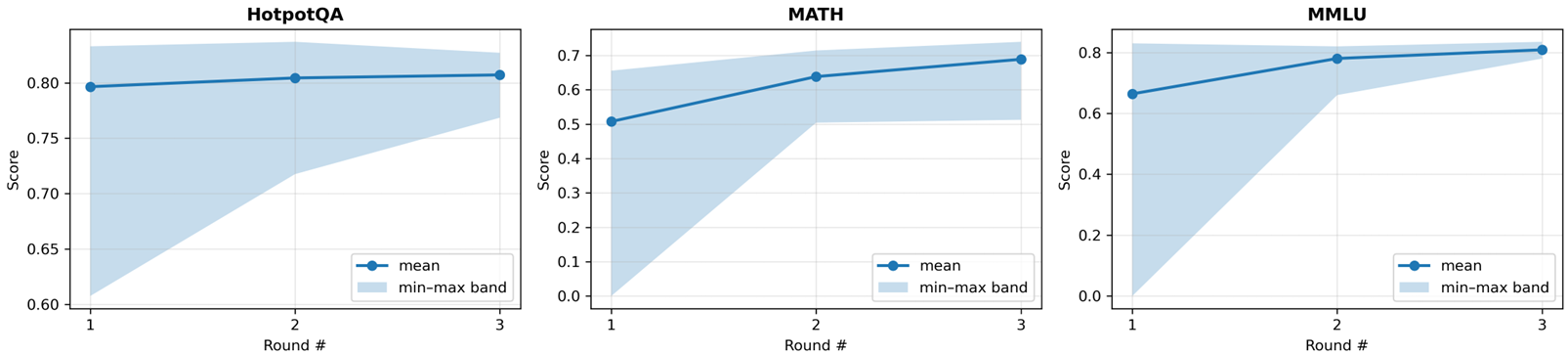}
  \caption{Evolution of validation accuracy across three sampling rounds.  Solid lines show the mean over $K=20$ samples; the shaded region spans the minimum and maximum run in each round.}
  % \zx{this figure looks bad, can you instead change it to 1x3 rather than 3x1? Or simply offload to Appendix}
  
  \label{fig:val_curves}
\end{figure}

\section{Helper Functions}\label{appendix:helper}

We present the function signatures of two helper functions which are exposed to the meta-optimizer LLM for workflow generation.

\onecolumn
\captionsetup{type=figure}
\begin{lstlisting}
def call_llm(
    messages: List[str],
    temperature: float,
    num_of_response: int,
    agent_role: str,
    instructions: str
) -> List[str]:
    """
    Call the ChatCompletion API and return a list of text responses.
    ----------
    messages : List[str]
        A list of input messages
    temperature : float
        Sampling temperature. Higher values (e.g., 0.8) yield more random outputs;
        lower values (e.g., 0.2) make the output more focused and deterministic.
    num_of_response : int
        Number of responses to generate from the model.
    agent_role : str
        The role/persona for the assistant, e.g., "helpful teacher". Included in the system prompt.
    instructions : str
        Additional task-specific guidance. Appended to the system prompt.
    Returns
    -------
    List[str]
        A list of generated response strings (one per requested completion).
    """
\end{lstlisting}
\captionof{figure}{Helper function on chat completion }
\vspace{1cm}

\captionsetup{type=figure}
\begin{lstlisting}
def exec_code(solution: str, entry_point: str) -> Union[str, List[Dict], Dict]:
    """
    This function tests a Python function implementation against public test cases, providing feedback about failures and execution errors.

    Args:
        solution (str): Complete Python code containing the function implementation.

        entry_point (str): Name of the function to test.

    Returns:
        Union[str, List[Dict], Dict]:
            - If ALL tests pass: returns the string "no error".
            - If some tests fail (AssertionError): returns List[Dict], each with:
                {
                  "test_case": str,        # the failing assertion
                  "error_type": "AssertionError",
                  "error_message": str,    # assertion message
                  "traceback": List[str]   # full traceback for debugging
                }
            - If code execution fails (syntax/runtime): returns Dict with:
                { "exec_fail_case": str }  # message describing the execution failure
    """
\end{lstlisting}
\captionof{figure}{Helper function on public test sets }

\section{Workflows from BayesFlow}\label{appendix:case_study}

\captionsetup{type=figure}
\begin{lstlisting}
class Workflow:
    def __init__(
        self,
        name: str,
        llm_config: dict,
        dataset: DatasetType,
    ) -> None:
        """
        Initialize the workflow with name, LLM configuration and dataset type.
        
        Args:
            name: Name of the workflow
            llm_config: Configuration for the LLM
            dataset: Type of dataset to use
        """
        self.name = name
        self.dataset = dataset
        self.llm = create_llm_instance(llm_config)

    async def __call__(self, problem: str, entry_point: str) -> Tuple[str, float]:
        """
        Execute the workflow to generate and test solutions.
        
        Args:
            problem: Problem description
            entry_point: Name of the function to implement
            
        Returns:
            Tuple containing the solution code and total cost
        """
        # Step 1: Generate initial solution
        messages = [f"Write a Python function named {entry_point} that solves this problem:\n{problem}"]
        solutions = await self.llm.call_llm(
            messages=messages,
            temperature=0.2,
            num_of_response=2,
            agent_role="expert Python programmer",
            instructions="Write clean, efficient code that exactly matches the function name and parameters specified. Return only the function implementation without explanations."
        )

        # Step 2: Test solutions
        test_results = []
        for solution in solutions:
            result = exec_code(solution, entry_point)
            test_results.append((solution, result))

        # Step 3: If any solution passes tests, return it
        for solution, result in test_results:
            if result == "no error":
                return solution, self.llm.get_usage_summary()["total_cost"]

        # Step 4: If no solution passes, try to fix the best solution
        best_solution = solutions[0]
        messages = [
            f"The following solution failed:\n{best_solution}\n\nError details:\n{test_results[0][1]}\n\nPlease fix the code to pass the tests.",
        ]
        fixed_solutions = await self.llm.call_llm(
            messages=messages,
            temperature=0.1,
            num_of_response=1,
            agent_role="expert Python debugger",
            instructions="Return only the corrected function implementation without explanations."
        )

        # Step 5: Test fixed solution
        for fixed_solution in fixed_solutions:
            result = exec_code(fixed_solution, entry_point)
            if result == "no error":
                return fixed_solution, self.llm.get_usage_summary()["total_cost"]

        # Step 6: If still no success, try one final attempt with more context
        messages = [
            f"Write a function named {entry_point} that solves this problem:\n{problem}\n\n"
            f"Previous attempts failed. Here are the test results:\n{test_results}\n\n"
            "Please implement a correct solution that passes all tests."
        ]
        final_solutions = await self.llm.call_llm(
            messages=messages,
            temperature=0.3,
            num_of_response=1,
            agent_role="senior Python developer",
            instructions="Focus on passing the test cases. Return only working code."
        )

        # Return the final attempt, whether it works or not
        return final_solutions[0], self.llm.get_usage_summary()["total_cost"] 
    \end{lstlisting}
\captionof{figure}{MBPP workflow }
\label{fig:mbpp_workflow}
\vspace{1cm}

\captionsetup{type=figure}
\begin{lstlisting}
class Workflow:
    def __init__(
        self,
        name,
        llm_config,
        dataset: DatasetType,
    ) -> None:
        self.name = name
        self.dataset = dataset
        self.llm = create_llm_instance(llm_config)

    async def __call__(self, problem: str) -> tuple[str, float]:
        """
        Implements a workflow for DROP dataset questions requiring discrete reasoning.
        Returns a tuple containing the final answer and total API cost.
        """
        # Step 1: Parse the question and identify key information
        parse_prompt = [
            f"Given this DROP dataset question, identify the type of reasoning required and key information:\n{problem}\n"
            "Focus on identifying whether it requires numerical reasoning, entity tracking, or date manipulation."
        ]
        analysis = await self.llm.call_llm(
            messages=parse_prompt,
            temperature=0.2,
            num_of_response=1,
            agent_role="analytical assistant",
            instructions="Analyze the question type and key information needed."
        )

        # Step 2: Generate detailed reasoning steps
        reasoning_prompt = [
            f"Question: {problem}\n"
            f"Analysis: {analysis[0]}\n"
            "Solve this step by step, showing your work clearly."
        ]
        reasoning_responses = await self.llm.call_llm(
            messages=reasoning_prompt,
            temperature=0.3,
            num_of_response=2,
            agent_role="mathematical expert",
            instructions="Show detailed step-by-step reasoning. Be precise with calculations and logic."
        )

        # Step 3: Cross-validate answers and extract final result
        validation_prompt = [
            f"Question: {problem}\n"
            f"Reasoning Path 1: {reasoning_responses[0]}\n"
            f"Reasoning Path 2: {reasoning_responses[1]}\n"
            "Compare these solutions and provide the final answer in its simplest form - just the number, name, or date without explanation."
        ]
        final_answers = await self.llm.call_llm(
            messages=validation_prompt,
            temperature=0.1,
            num_of_response=1,
            agent_role="precise evaluator",
            instructions="Extract only the final answer in its simplest form - no explanations or additional text."
        )

        # Step 4: Format the answer according to requirements
        answer = final_answers[0].strip()
        # Remove any explanatory text or prefixes
        if "answer is" in answer.lower():
            answer = answer.split("answer is")[-1]
        answer = answer.strip(" .,:")
        
        # Return the final answer and total cost
        return answer, self.llm.get_usage_summary()["total_cost"] 
    \end{lstlisting}
\captionof{figure}{DROP workflow }
\vspace{1cm}

\captionsetup{type=figure}
\begin{lstlisting}
class Workflow:
    def __init__(
        self,
        name,
        llm_config,
        dataset: DatasetType,
    ) -> None:
        self.name = name
        self.dataset = dataset
        self.llm = create_llm_instance(llm_config)

    async def __call__(self, problem: str) -> tuple[str, float]:
        """
        Implements a workflow for GPQA multiple-choice questions.
        Returns:
        - answer: str - Single capital letter (A-D) representing the chosen option
        - cost: float - Total cost of LLM API calls
        """
        # Step 1: Domain-specific knowledge retrieval with emphasis on quantitative aspects
        knowledge_prompt = f"""
        For this graduate-level science question:
        {problem}

        1. Identify the specific scientific domain and sub-topic
        2. List the key scientific principles and mathematical relationships
        3. Write out ALL relevant formulas, equations, and units
        4. Specify numerical ranges, constants, or threshold values
        5. Note critical assumptions and boundary conditions
        6. Identify potential calculation pitfalls or unit conversion issues

        Structure your response with clear headers and numbered equations.
        """
        knowledge_responses = await self.llm.call_llm(
            messages=[knowledge_prompt],
            temperature=0.2,
            num_of_response=2,
            agent_role="domain expert",
            instructions="Provide comprehensive quantitative knowledge relevant to the question."
        )

        # Step 2: Mathematical analysis and calculations
        math_prompt = f"""
        Using the established knowledge bases:
        Knowledge Base 1: {knowledge_responses[0]}
        Knowledge Base 2: {knowledge_responses[1]}

        For the question:
        {problem}

        1. Set up the mathematical framework:
           - Define all variables and their units
           - List equations to be used
           - Identify given values and unknowns
        2. Perform step-by-step calculations:
           - Show all work clearly
           - Include unit conversions
           - Note intermediate results
        3. Calculate numerical results for each option if applicable
        4. Estimate reasonable ranges for answers

        Structure as a detailed mathematical solution.
        """
        math_responses = await self.llm.call_llm(
            messages=[math_prompt],
            temperature=0.2,
            num_of_response=2,
            agent_role="mathematical physicist",
            instructions="Show detailed mathematical work and calculations."
        )

        # Step 3: Scientific reasoning with quantitative support
        analysis_prompt = f"""
        Based on the mathematical analysis:
        Calculation Path 1: {math_responses[0]}
        Calculation Path 2: {math_responses[1]}

        For the question:
        {problem}
        
        1. Evaluate how each calculation approach supports/contradicts the options
        2. Consider numerical accuracy and significant figures
        3. Check for mathematical consistency with physical principles
        4. Identify which calculations most definitively distinguish between options
        """
        analysis_responses = await self.llm.call_llm(
            messages=[analysis_prompt],
            temperature=0.3,
            num_of_response=2,
            agent_role="expert scientist",
            instructions="Analyze calculations and their implications for each option."
        )

        # Step 4: Generate answers with quantitative justification
        solution_prompt = f"""
        Using the mathematical results and analysis:
        Analysis 1: {analysis_responses[0]}
        Analysis 2: {analysis_responses[1]}
        
        Choose the correct answer (A, B, C, or D) and explain:
        1. How the calculations support this choice
        2. Why the numerical results rule out other options
        3. Any mathematical constraints that confirm this answer

        Question: {problem}
        """
        solution_responses = await self.llm.call_llm(
            messages=[solution_prompt],
            temperature=0.2,
            num_of_response=3,
            agent_role="scientific expert",
            instructions="Select and justify answer based on quantitative evidence."
        )

        # Step 5: Final verification with emphasis on mathematical consistency
        verification_prompt = f"""
        Review these solutions:
        Solution 1: {solution_responses[0]}
        Solution 2: {solution_responses[1]}
        Solution 3: {solution_responses[2]}

        Question: {problem}
        
        Verify:
        1. Mathematical correctness
        2. Consistency with physical laws
        3. Proper handling of units and significant figures
        
        Output only a single capital letter (A, B, C, or D) representing the final answer.
        """
        final_response = await self.llm.call_llm(
            messages=[verification_prompt],
            temperature=0.1,
            num_of_response=1,
            agent_role="scientific reviewer",
            instructions="Output only a single capital letter (A, B, C, or D) as the final answer."
        )

        # Extract the single letter answer
        answer = final_response[0].strip()
        if len(answer) > 1:
            answer = ''.join(char for char in answer if char in 'ABCD')[:1]

        return answer, self.llm.get_usage_summary()["total_cost"] 
    \end{lstlisting}
\captionof{figure}{GPQA workflow }
\vspace{1cm}

\captionsetup{type=figure}
\begin{lstlisting}

import re
import statistics
class Workflow:
    def __init__(
        self,
        name,
        llm_config,
        dataset: DatasetType,
    ) -> None:
        self.name = name
        self.dataset = dataset
        self.llm = create_llm_instance(llm_config)

    async def __call__(self, problem: str) -> tuple[str, float]:
        """
        Implements a multi-step workflow for solving GSM8K math problems.
        Returns the final numeric answer in LaTeX boxed format and total API cost.
        """
        # Step 1: Generate detailed step-by-step solutions
        solution_prompt = [
            f"Solve this math problem step by step:\n{problem}\n"
            "Show your work clearly, and end with the final numeric answer only (no units or symbols)."
        ]
        solutions = await self.llm.call_llm(
            messages=solution_prompt,
            temperature=0.2,  # Low temperature for precise calculations
            num_of_response=2,  # Generate two solutions as verification
            agent_role="math teacher",
            instructions="Break down the problem into clear steps. Verify calculations carefully."
        )

        # Step 2: Extract numeric answers from solutions
        numbers = []
        for solution in solutions:
            matches = re.findall(r'-?\d*\.?\d+', solution.split('\n')[-1])
            if matches:
                numbers.append(float(matches[0]))
        
        # Step 3: Verify answers with a different approach
        verify_prompt = [
            f"Question: {problem}\n"
            "Solve this problem using a different method than before. "
            "Give only the final numeric answer without any units."
        ]
        verification = await self.llm.call_llm(
            messages=verify_prompt,
            temperature=0.3,
            num_of_response=1,
            agent_role="math expert",
            instructions="Focus on accuracy. Double-check your calculations."
        )

        # Step 4: Extract verification number
        verify_num = None
        if verification:
            matches = re.findall(r'-?\d*\.?\d+', verification[0].split('\n')[-1])
            if matches:
                verify_num = float(matches[0])
                numbers.append(verify_num)

        # Step 5: Determine final answer through consensus
        if not numbers:
            # Fallback if no numbers extracted
            final_prompt = [
                f"What is the final numeric answer to this problem? {problem}\n"
                "Give ONLY the number, no explanation or units."
            ]
            final_attempt = await self.llm.call_llm(
                messages=final_prompt,
                temperature=0.1,
                num_of_response=1,
                agent_role="math expert",
                instructions="Provide only the final numeric answer."
            )
            matches = re.findall(r'-?\d*\.?\d+', final_attempt[0])
            final_answer = float(matches[0]) if matches else 0.0
        else:
            # Use median to handle outliers
            final_answer = statistics.median(numbers)

        # Step 6: Format answer in LaTeX boxed notation
        formatted_answer = f"\\boxed{{{int(final_answer) if final_answer.is_integer() else final_answer}}}"
        
        return formatted_answer, self.llm.get_usage_summary()["total_cost"] 
    \end{lstlisting}
\captionof{figure}{GSM8K workflow }
\vspace{1cm}

\captionsetup{type=figure}
\begin{lstlisting}
class Workflow:
    def __init__(
        self,
        name,
        llm_config,
        dataset: DatasetType,
    ) -> None:
        self.name = name
        self.dataset = dataset
        self.llm = create_llm_instance(llm_config)

    async def __call__(self, problem: str) -> tuple[str, float]:
        """
        Implements a multi-step workflow for HotpotQA with reasoning verification.
        Returns:
        - answer: str - The final answer only, without reasoning or explanation
        - cost: float - Total cost of LLM API calls
        """
        # Step 1: Initial analysis to identify key information and question type
        analysis_prompt = f"Analyze this multi-hop question and identify the key entities and relationships needed: {problem}"
        analysis = await self.llm.call_llm(
            messages=[analysis_prompt],
            temperature=0.3,
            num_of_response=1,
            agent_role="analytical assistant",
            instructions="Identify the key entities, relationships, and the type of information needed to answer this question."
        )

        # Step 2: Extract and validate supporting facts from Wikipedia context
        facts_prompt = f"Question: {problem}\nAnalysis: {analysis[0]}\nIdentify the specific Wikipedia sentences that contain the essential information needed to answer this question. For each fact, explain why it's necessary for the multi-hop reasoning chain."
        supporting_facts = await self.llm.call_llm(
            messages=[facts_prompt],
            temperature=0.4,
            num_of_response=3,
            agent_role="evidence finder",
            instructions="Extract only the most relevant sentences from the Wikipedia context. Ensure each fact is necessary for bridging the reasoning steps. Identify any missing information."
        )

        # Step 3: Generate multiple reasoning paths using validated supporting facts
        reasoning_paths = []
        for facts in supporting_facts:
            reasoning_prompt = f"Question: {problem}\nRelevant Facts: {facts}\nProvide a step-by-step reasoning path that connects these supporting facts to answer the question."
            paths = await self.llm.call_llm(
                messages=[reasoning_prompt],
                temperature=0.7,
                num_of_response=2,
                agent_role="logical reasoner",
                instructions="Create a clear reasoning chain that shows how the supporting facts connect to reach the answer. Each step should be grounded in the provided facts."
            )
            reasoning_paths.extend(paths)

        # Step 4: Generate candidate answers based on each reasoning path
        answers = []
        for path in reasoning_paths:
            answer_prompt = f"Question: {problem}\nReasoning: {path}\nProvide only the final answer in its simplest form without explanation."
            candidate = await self.llm.call_llm(
                messages=[answer_prompt],
                temperature=0.2,
                num_of_response=1,
                agent_role="precise answerer",
                instructions="Give only the exact answer in its simplest form - just the name, number, or fact requested. No explanations."
            )
            answers.extend(candidate)

        # Step 5: Verify and select the final answer
        verification_prompt = f"Question: {problem}\nSupporting Facts: {supporting_facts}\nCandidate answers: {answers}\nSelect the most accurate answer that is fully supported by the facts and simplify it to its most basic form."
        final_answer = await self.llm.call_llm(
            messages=[verification_prompt],
            temperature=0.1,
            num_of_response=1,
            agent_role="answer validator",
            instructions="Choose and simplify the answer to its most basic form, ensuring it is fully supported by the extracted Wikipedia facts. For names, give just the name. For yes/no questions, respond only with Yes or No. For measurements, give just the number and unit. Remove any explanations or additional context."
        )

        # Return the simplified final answer and total cost
        return final_answer[0], self.llm.get_usage_summary()["total_cost"] 
    \end{lstlisting}
\captionof{figure}{HotpotQA workflow }
\vspace{1cm}

\captionsetup{type=figure}
\begin{lstlisting}
class Workflow:
    def __init__(
        self,
        name,
        llm_config,
        dataset: DatasetType,
    ) -> None:
        self.name = name
        self.dataset = dataset
        self.llm = create_llm_instance(llm_config)

    async def __call__(self, problem: str) -> tuple[str, float]:
        """
        Implements a workflow for solving MATH problems with multiple verification steps.
        Returns:
        - answer: str - Final answer in LaTeX boxed format
        - cost: float - Total cost of LLM API calls
        """
        # Step 1: Plan solution approach and identify key concepts
        plan_instructions = """
        Analyze this math problem and create a solution plan:
        1. Identify the mathematical concepts and theorems needed
        2. Break down the problem into clear logical steps
        3. Note any potential pitfalls or edge cases to consider
        4. Outline the sequence of calculations needed
        Do not solve the problem yet - focus on planning the approach.
        """
        solution_plans = await self.llm.call_llm(
            messages=[problem],
            temperature=0.4,
            num_of_response=2,
            agent_role="mathematical strategist",
            instructions=plan_instructions
        )

        # Step 2: Execute solution with detailed reasoning using the plans
        solve_instructions = f"""
        Solve this math problem following the solution plans provided:
        Plan 1: {solution_plans[0]}
        Plan 2: {solution_plans[1]}
        
        Show your work step by step, following these guidelines:
        1. Use the identified concepts and theorems correctly
        2. Follow the planned logical steps
        3. Watch for the noted pitfalls
        4. Double-check calculations
        5. Format the final answer using LaTeX boxed{{}}
        """
        solutions = await self.llm.call_llm(
            messages=[f"Problem: {problem}"],
            temperature=0.3,
            num_of_response=2,
            agent_role="expert mathematician",
            instructions=solve_instructions
        )

        # Step 3: Verify and validate solutions
        verify_instructions = """
        Verify these solutions for correctness:
        1. Check mathematical logic and steps
        2. Verify calculations
        3. Ensure answer format is correct (LaTeX boxed{})
        4. Compare the two solutions
        Return only the correct boxed answer. If solutions differ, explain why one is correct.
        """
        verification = await self.llm.call_llm(
            messages=[f"Problem: {problem}\nSolution 1: {solutions[0]}\nSolution 2: {solutions[1]}"],
            temperature=0.2,
            num_of_response=1,
            agent_role="mathematical reviewer",
            instructions=verify_instructions
        )

        # Step 4: Format check and extraction
        format_instructions = """
        Extract and format the final answer following these rules:
        1. Answer must be in LaTeX boxed{} format
        2. For fractions, use \frac{n}{d}
        3. For multiple values, separate with commas inside boxed{}
        4. For vectors/matrices, use proper LaTeX notation
        5. Remove any explanation text, keep only the boxed{} answer
        Return only the formatted answer.
        """
        formatted_answer = await self.llm.call_llm(
            messages=[verification[0]],
            temperature=0.1,
            num_of_response=1,
            agent_role="LaTeX formatter",
            instructions=format_instructions
        )

        # Extract just the boxed answer
        answer = formatted_answer[0]
        if "\\boxed{" not in answer:
            # Final fallback to ensure boxed format
            answer = f"\\boxed{{{answer}}}"

        return answer, self.llm.get_usage_summary()["total_cost"] 
    \end{lstlisting}
\captionof{figure}{MATH workflow }
\vspace{1cm}

\captionsetup{type=figure}
\begin{lstlisting}
class Workflow:
    def __init__(
        self,
        name,
        llm_config,
        dataset: DatasetType,
    ) -> None:
        self.name = name
        self.dataset = dataset
        self.llm = create_llm_instance(llm_config)

    async def __call__(self, problem: str) -> tuple[str, float]:
        """
        Implements a workflow for MMLU multiple choice questions.
        Returns:
        - answer: str - A single capital letter (A-J) representing the chosen option
        - cost: float - Total cost of LLM API calls
        """
        # Step 1: Structured decomposition of the question
        decomposition_prompt = f"""
        Break down this multiple choice question into its core components:
        {problem}

        1. Core Concept: What is the main topic or principle being tested?
        2. Given Information: What key facts or conditions are provided?
        3. Required Reasoning: What logical steps are needed to connect the given information to the answer?
        4. Relationships: How do different elements in the question relate to each other?
        5. Answer Criteria: What specific characteristics must the correct answer satisfy?

        Provide a structured analysis addressing each point above.
        """
        decomposition_responses = await self.llm.call_llm(
            messages=[decomposition_prompt],
            temperature=0.2,
            num_of_response=2,
            agent_role="analytical expert",
            instructions="Break down the question systematically, focusing on logical relationships and key components."
        )

        # Step 2: Detailed analysis incorporating decomposition insights
        analysis_prompt = f"""
        Using this structural analysis:
        {decomposition_responses[0]}
        {decomposition_responses[1]}

        Evaluate the multiple choice question:
        {problem}

        For each option (A-J):
        1. How well does it align with the core concept identified?
        2. Does it satisfy the answer criteria established?
        3. Is it supported by the logical relationships we identified?

        After analysis, select the most appropriate answer choice (A-J).
        """
        analysis_responses = await self.llm.call_llm(
            messages=[analysis_prompt],
            temperature=0.3,
            num_of_response=2,
            agent_role="expert academic advisor",
            instructions="Use the structural analysis to systematically evaluate each option and clearly indicate your final answer as a single capital letter A-J."
        )

        # Step 3: Cross-check against decomposition framework
        verification_prompt = f"""
        Question: {problem}

        Consider how each answer option aligns with our structural analysis:
        {decomposition_responses[0]}

        Previous answer suggestions: {', '.join(analysis_responses)}

        Verify the answer by:
        1. Testing it against each component of our decomposition
        2. Checking for logical consistency with identified relationships
        3. Confirming it meets all answer criteria

        Which option (A-J) best satisfies these requirements?
        """
        verification_responses = await self.llm.call_llm(
            messages=[verification_prompt],
            temperature=0.2,
            num_of_response=1,
            agent_role="subject matter expert",
            instructions="Systematically verify the answer against our structural analysis framework."
        )

        # Step 4: Final consensus with strict constraints
        all_responses = analysis_responses + [verification_responses[0]]
        consensus_prompt = f"""
        Question: {problem}

        Previous analyses suggested these answers: {', '.join(all_responses)}

        Based on our complete structural analysis and verification, determine the final answer.
        You must output EXACTLY one capital letter A-J, nothing else.
        """
        
        final_response = await self.llm.call_llm(
            messages=[consensus_prompt],
            temperature=0.1,
            num_of_response=1,
            agent_role="decision maker",
            instructions="Output only a single capital letter A-J as the final answer, with no additional text."
        )

        # Step 5: Extract and validate the answer
        answer = re.findall(r'[A-J]', final_response[0])[0]
        
        return answer, self.llm.get_usage_summary()["total_cost"] 
    \end{lstlisting}
\captionof{figure}{MMLU-Pro workflow}

\twocolumn

\end{document}